\documentclass[11pt]{article} 

\usepackage{url}
\usepackage{epsf,epstopdf}
\usepackage{comment}
\usepackage{times}
\usepackage{amsxtra,amscd,bm}
\usepackage{enumitem}
\setlist[itemize]{leftmargin=0mm}
\usepackage{microtype}
\usepackage{graphicx}
\usepackage{subcaption}
\usepackage{amsmath,amsbsy,amsfonts,amssymb,amsthm,bm}
\usepackage{algorithm,algorithmic,mathtools}
\usepackage{color,cases,multirow}
\usepackage[margin=1.2in]{geometry} 
\usepackage[colorlinks=True, citecolor=blue]{hyperref}
\usepackage{todonotes}
\usepackage{authblk}

\usepackage{amsmath,amsbsy,amsfonts,amssymb,amsthm,bm,color}
\newtheorem{theorem}{Theorem}
\newtheorem{remark}{Remark}
\newtheorem{proposition}[theorem]{Proposition}
\newtheorem{lemma}{Lemma}

\newtheorem{corollary}{Corollary}
\theoremstyle{definition}
\newtheorem{definition}{Definition}

\newcommand{\cH}{\mathcal{H}}
\newcommand{\cP}{\mathcal{P}}
\newcommand{\tP}{\tilde{\mathcal{P}}}
\newcommand{\cF}{\mathcal{F}}

\newcommand{\cB}{\mathcal{B}}

\newcommand{\cC}{\mathcal{C}}

\newcommand{\balpha}{\bm{\alpha}}

\newcommand{\bb}{\bm{b}}

\newcommand{\be}{\bm{e}}
\newcommand{\bh}{\bm{h}}
\newcommand{\bg}{\bm{g}}
\newcommand{\bM}{\bm{M}}
\newcommand{\bn}{\bm{n}}

\newcommand{\bx}{\bm{x}}

\newcommand{\bu}{\bm{u}}

\newcommand{\bw}{\bm{w}}
\newcommand{\bxi}{\bm{\xi}}
\newcommand{\rad}{\text{Rad}}

\newcommand{\bR}{\mathbb{R}}
\newcommand{\EE}{\mathbb{E}}

\providecommand{\keywords}[1]{\textbf{\textit{Keywords---}} #1}

\title{Complexity Measures for Neural Networks with General Activation Functions Using Path-based Norms}

\author[1]{Zhong Li \thanks{\texttt{li\_zhong@pku.edu.cn}}}
\author[2]{Chao Ma \thanks{\texttt{chaoma@stanford.edu}}}
\author[3]{Lei Wu \thanks{\texttt{leiwu@princeton.edu}}}

\affil[1]{School of Mathematical Sciences, Peking University}
\affil[2]{Department of Mathematics, Stanford University}
\affil[3]{Program in Applied and Computational Mathematics, Princeton University}

\begin{document}

\maketitle

\begin{abstract}
A simple approach is proposed to  obtain complexity controls for neural networks with general activation functions. The approach is motivated by approximating the general activation functions  with  one-dimensional ReLU networks, which reduces the problem to the  complexity controls of ReLU networks.
Specifically, we consider two-layer networks and deep residual networks, for which path-based norms are derived to control complexities. We also provide preliminary analyses of the function spaces induced by these norms and a priori estimates of the corresponding regularized estimators. 
\end{abstract}


\section{Introduction}

Norm-based complexity measures have played an essential role in analyzing the generalization ability of machine learning models. For  example, a variety of norms have been explored for linear models, such as the $\ell_1$ norm for LASSO and $\ell_2$ norm for ridge regression. A good choice of the norm can help design good regularization strategies and provide accurate estimate of the model's error on test data. 
In deep learning, norm-based complexity measures are especially crucial, since neural networks used in practice are always over-parameterized, in which case measures depending on the network size inevitably lead to vacuous generalization bounds. In contrast, norm-based measures are able to provide size-independent controls of the model's complexity. 

The work~\cite{neyshabur2015norm} obtained the first norm-based controls for multilayer fully-connected networks with ReLU activation functions via pure Rademacher complexity-based analysis. \cite{ma2019priori} later generalized the result to deep residual networks.  \cite{barron2018approximation} provided a covering number-based approach to obtain similar results for fully connected networks. 
All the complexity measures identified in these work are certain path-based norms. \cite{e2018priori,ma2019priori,e2019barron,ma2019generalization} further developed the approximation and estimation theories for function spaces induced by these norms. However, all these results are limited to activation functions with the positive homogeneity property, such as ReLU and leaky ReLU~\cite{He2015b}. 
In this work, we aim to extend these results to neural networks with general activation functions, since activation functions other than ReLU and leaky-ReLU (e.g. Sigmoid, Tanh) are also widely used in practice.

% As comparison,  \cite{bartlett2017nips, pmlr-v75-golowich18a}  considered complexity controls for neural networks with general activation functions, which only utilize the Lipschitz property of activation functions. However the complexity measures identified are  in the form of $\prod \|W_i\|$, i.e. the product of the norm of each layer's weight matrix, which is usually much worse than $\|\prod |W_i|\|$, i.e. the path-based norms for ReLU networks.  

Our approach to deal with general activation functions can be decomposed into two steps. Firstly, we approximate the  activation function by a two-layer ReLU network.  Thus the original  network is converted to a ReLU network but with more parameters. Secondly, we apply the path-norm complexity control to the induced ReLU networks, then we obtain a new path-based norm for the original networks. This norm can be used to control the complexity of the original networks.

\subsection{Our contribution}

We explicitly write down the path-based norms for two-layer neural networks and deep residual networks. Moreover,  the estimates of the Rademacher complexity based on these norms are also provided.  These path-based norms in general consist of two terms. The first one is the same as the path norm for the corresponding ReLU networks. The second one is an additional term arising from the approximation of activation functions.  Notably, different from the original path norm which only depends on the full paths, the new path-based norm relies on all the paths with different lengths, including the short ones starting from intermediate layers. 

% Moreover, these modified path-based norms are equivalent to the original path-based norm if the activation function is ReLU.

We also provide some  results for  the a priori estimates for the generalization error of the regularized estimator
    \[
        \hat{\theta}_n := \text{argmin}_{\theta}\, \hat{R}_n(\theta) + \lambda \|\theta\|,
    \]
where $\hat{R}_n$ is the empirical risk and $\|\theta\|$ is the norm acting as the complexity measure. These results can be viewed as the extension of the previous work on function spaces and a priori estimates of ReLU neural networks~\cite{e2018priori,ma2019priori,e2019barron,ma2019generalization}.

Beside the main results, we also provide the characterization of the approximation ability of one-dimensional ReLU networks with finite path norm, which may be of independent interest to readers. 
Roughly speaking, we show that an one-dimensional function $f$ can be approximated in $L^{\infty}(\bR)$ by two-layer ReLU networks with bounded path norm as long as  
\begin{equation}\label{gammadef}
\gamma_0(f) := \int_\bR |f''(x)|(|x|+1)dx < +\infty.
\end{equation}
Moreover, the ReLU network's path norm can be roughly bounded by $\gamma_0(f)$ (see Theorem \ref{1DApp} for details).  A similar characterization appeared in  \cite{savarese19a}, but neglected the bias terms.

\subsection{Related work}

\cite{bartlett2017nips,pmlr-v75-golowich18a,neyshabur2018a} also provided complexity controls for neural networks with general activation functions, but in a layer-wise fashion. For example, \cite{bartlett2017nips} considered the fully connected network $f(\bx;\theta)=A_L\sigma(A_{L-1}\sigma(\dots \sigma(A_1 \bx)))$ controlled in the following way
\[
    \cH_{\bm{s},\bb} :=\{f(\cdot;\theta): \theta=(A_1,\dots,A_L), \|A_i\|_{2}\leq s_i, \|A^T_i-M_i^T\|_{2,1}\leq b_i\},
\]
where $\{M_i\}$ are fixed reference matrices, and  $\bm{s}=(s_1,\dots,s_L), \bb=(b_1,\dots,b_L)$. In the context of norm-based measures, the layer-wise control inevitably leads to measures in terms of the product of norms of weight matrices of each layer  $\prod \|A_i\|$.  
However, we are concerning a global control of hypothesis space
\[
\cF_C := \{f(\cdot;\theta): \|\theta\|\leq C\},
\]
where $\|\theta\|$ is a  norm of $\theta$.  \cite{neyshabur2015norm,barron2018approximation,ma2019priori} showed that for positive homogeneous networks, this type of control can produce path-based complexity measures, which usually appear in the form of $\|\prod |A_i|\|$, where $|A_i|$ denotes the matrix obtained by taking entry-wise absolute values for $A_i$.  Usually this type of norms are better than $\prod \|A_i\|$ \cite{theisen2019global}.
In this paper, we extend this type of results to the case of general activation functions.

Another closely related work is \cite{jason2019general}, which also obtained the path-based complexity measures for neural networks with general activation functions. The distinguishes between our work and theirs are given as follows. \cite{jason2019general} considered fully connected networks (including two-layer networks), while we considered two-layer networks and deep residual networks. The limits of two-layer and deep residual networks as network size goes to infinity are well-defined \cite{e2019barron}, so we also study the approximation function spaces induced by the path-based norms.  In contrast, the limit of multilayer fully connected networks are still unclear now \cite{nguyen2019mean,sirignano2019deep}. Moreover, the techniques used are different. \cite{jason2019general} used the sampling-based approach to obtain the bounds of covering numbers, which cannot be directly applied to neural networks with skip-connections \cite{theisen2019global}. However, our approach is more general, which can be applied to any architectures as long as the complexity controls of corresponding ReLU networks exist. 

It is worth mentioning that the idea of approximating activation functions with a small network was also exploited in the literature on network in network \cite{lin2013NIN}, which proposed to replace the simple activation function with a complex ``micro network'' to enhance model's expressivity.

\paragraph{Notation.} 
Throughout this paper, let $[n]=\{1,2,\dots,n\}$, if $n$ is a positive integer. We use $\|\cdot\|_2$ and $\|\cdot\|_F$ to denote the $\ell_2$ and Frobenius norms for matrices, respectively.  We let $\mathbb{S}^{d-1}=\{\bx\,:\,\|\bx\|=1\}$.  We use $X\lesssim Y$ to indicate that there exists an absolute constant $C_0>0$ such that $X\leq C_0Y$, and  $X\gtrsim Y$ is similarly defined. For any matrix $A=(a_{i,j})$, denote by $|A|=(|a_{i,j}|)$. For any $\bx\in\bR^d$, let $\tilde{\bx} = (\bx^T,1)^T\in\bR^{d+1}$.

\section{Preliminaries}

In this work, we consider the standard supervised learning setup. Let $S=\{(\bx_i,y_i)\}_{i=1}^n$ denote $n$ samples with $y_i=f^*(\bx_i)+\varepsilon_i$. We always assume that $\bx\in X:=[-1,1]^d$ and $f^*(\bx)\in [0,1]$. The noise $\{\varepsilon_i\}$ are i.i.d. random variables that satisfy $\EE[\varepsilon_i]=0$ and $\EE[\varepsilon_i^2]<\infty$.

Let $f(\bx;\theta)$ denote the parametric model. Consider the truncated square loss
\begin{align}
   \ell(\bx,y;\theta)=\frac{1}{2}\left(\mathcal{T}_{[0,1]}f(\bx;\theta)-y\right)^2,
\end{align}
where $\mathcal{T}_{[0,1]}$ is the truncation operator such that $\mathcal{T}_{[0,1]}g(\bx)=\min\{\max\{g(\bx),0\},1\}$ for any function $g: \bR^d\mapsto \bR$.
Then the population risk and empirical risk are defined as 
\begin{align}
   R(\theta)=\mathbb{E}_{\bx,y} [\ell(\bx,y;\theta)], \ \quad \hat{R}_n(\theta)=\frac{1}{n}\sum_{i=1}^n \ell(\bx_i,y_i;\theta).
\end{align}
The difference between two risks is called the generalization gap. 

For a function class $\cF$, the (empirical) Rademacher complexity \cite{shalev2014understanding} with respect to the data set $S$ is defined as 
\begin{equation}
\rad_n(\mathcal{F})=\frac{1}{n}\mathbb{E}_{\bxi}[\sup_{f\in\mathcal{F}}\sum_{i=1}^n\xi_if(\bx_i)],
\end{equation}
where the $\{\xi_i\}_{i=1}^n$ are independent random variables with $\mathbb{P}(\xi_i=1)=\mathbb{P}(\xi_i=-1)=1/2$.

A two-layer neural network is given by 
\begin{equation}\label{eqn: two-layer-net}
f_m(x;\theta)=\sum\limits_{k=1}^m a_k\sigma(\bb_k^T \bx+c_k),
\end{equation}
where $\sigma: \bR\mapsto\bR$ is a nonlinear activation function, and $\theta=\{(a_k,\bb_k,c_k)\}_{k=1}^m$ denote the parameters to be learned from the training data. For the ReLU networks, i.e. $\sigma(t)=\sigma_R(t):=\max(0,t)$, we define its path norm \cite{neyshabur2015norm} by 
\begin{equation}\label{CmPthNrm2NN}
    \|\theta\|_{\cP} := \sum_{k=1}^m |a_k|(\|\bb_k\|_1+|c_k|).
\end{equation} 
Throughout this paper, we will use $\sigma_R$ to denote ReLU function.

\section{Approximating one-dimensional functions by two-layer ReLU networks}\label{sec: one-dimensional}

We begin with the characterization of one-dimensional functions that can be approximated by two-layer ReLU networks with bounded path norms. This result will serve as the cornerstone for the  analysis of neural networks with general activation functions in the following sections.

\begin{theorem}\label{1DApp}
Consider the function $f: \bR\mapsto\bR$. Assume that 
$f(x)$ is continuous and twice weakly differentiable on $\bR$, and $f''(x)$ is locally Riemann integrable on $\bR$.\footnote{For any function $h(x)$ defined on $\bR$, $h$ is locally Riemann integrable means that $h$ is Riemann integrable on any compact subset of $\bR$.} Define
\begin{align}
 \gamma_0(f) &= \int_\bR |f''(x)|(|x|+1)dx, \label{gam0} \\
 g(x)&=|f(x)|+(|x|+2)|f'(x)|,\label{g01}
\end{align}
and
\begin{equation}\label{gam}
\gamma(f) = \gamma_0(f)+\inf_{x\in \bR}g(x).
\end{equation}
If $\gamma(f)<+\infty$, then for any $\epsilon>0$, there exists a two-layer ReLU neural network $f_m(\cdot;\theta)$ of width $m<+\infty$, such that
\begin{align}
\sup_{x\in\bR}|f(x)-f_m(x;\theta)| &\le \epsilon,\\
\|\theta\|_{\cP}&\leq\gamma(f)+\epsilon.
\end{align}
\end{theorem}

Here, the derivatives $f'$ and $f''$ should be understood in the weak sense. The complete proof is deferred to Appendix \ref{sec: 1d-approx}. Notice that for any $h(x)=a+bx$ with $a,b\in\bR$, we have  $\gamma_0(f+h) = \gamma_0(f)$. It implies that adding a linear part does not change the value of $\gamma_0(\cdot)$, but the two-layer neural networks used for approximation must change accordingly. 
The extra term $\inf_{x\in \bR}g(x)$ is introduced to account for the linear part of $f$.

Theorem \ref{1DApp} implies that the ``norm'' $\gamma(\cdot)$ is a good measure to characterize whether an activation function can be approximated by two-layer ReLU network with bounded path norm.
As a comparison, \cite{savarese19a} provided a similar characterization as follows 
\begin{equation}\label{eqn: one-relu-old}
    \max\left\{\int_{\bR} |f''(x)| dx, |f'(+\infty)+f'(-\infty)|\right\} < +\infty
\end{equation}
for two-layer ReLU networks with  $\sum_{k=1}^m |a_k||b_k|$ bounded. Apparently, the condition \eqref{eqn: one-relu-old} is weaker than $\gamma(f)<+\infty$. But it neglects the influence of the bias term $c_k$, which is crucial for providing complexity control.

To prove Theorem $\ref{1DApp}$, we need the following lemma, whose proof can be found in Appendix \ref{sec:prf-Asymptote}.
\begin{lemma}\label{Asymptote}
Let $f$ satisfies $\gamma_0(f)<+\infty$, then there exist constants $a$, $b$, $c$ and $d$, such that
\begin{equation}
\lim_{x\rightarrow-\infty}|f(x)-(ax+b)|=0,\qquad \lim_{x\rightarrow+\infty}|f(x)-(cx+d)|=0.
\end{equation}
\end{lemma}

The above lemma suggests that our method can only deal with activation functions that possess linear asymptotes. This property is  surprisingly satisfied  by  all the commonly used activation functions. We calculate the $\gamma(\cdot)$ norms for common activation functions and the results are shown in Table \ref{tab: activations}. We see that all functions considered in Table~\ref{tab: activations} have finite $\gamma(\cdot)$ norms.

\begin{table}[!h]
\renewcommand\arraystretch{1.2}
\centering
\caption{
The $\gamma(\cdot)$ norms of commonly used activation functions. Here ELU is the exponential linear unit \cite{Clevert2015}; LReLU is the leaky ReLU; GELU is the Gaussian Error Linear Unit \cite{hendrycks2016gaussian}. Swish is the activation function discovered by reinforcement learning-based searches \cite{ramachandran2017searching}.
The $\alpha$ is the hyper-parameter which appears in the corresponding activation function. The detail definitions of these activation functions and the calculations of $\gamma(\cdot)$ norm are deferred to  Appendix \ref{sec: cacl-activation}.
}
\label{tab: activations}
\begin{tabular}{c|c|c|c|c|c|c|c|c} 
\hline 
$\sigma$ &ReLU &Sigmoid & Tanh &  ELU & LReLU &  GELU  & Softplus & Swish \\\hline\hline
$\gamma(\sigma)$& 1 & 1.5 & 5 & $3|\alpha|$+1 & $\alpha+1$ & $\approx$ 2.7& $1+2\ln 2$ & $\approx \frac{1.8}{\alpha}+1.4$  \\
\hline 
 \end{tabular}

\end{table}

\section{Two-layer neural networks}\label{sec: two-layer}

Now we utilize Theorem $\ref{1DApp}$ to derive an upper bound for the Rademacher complexity of two-layer neural network with general activation functions. We first need the result for two-layer ReLU networks, whose proof can be found in \cite{neyshabur2015norm} and the appendix of \cite{e2018priori}.
\begin{proposition}\label{pro: two-relu-layer}
Let $f^0_m(\cdot;\theta)$ denote the two-layer ReLU networks. Let $\cF_Q^0 := \{ f^0_m(\cdot;\theta) : \|\theta\|_{\cP}\leq Q, m\in \mathbb{N}_{+}\}$. Then we have 
\begin{equation}\label{eqn: rad-relu}
    \rad_n(\cF_Q^0) \leq 2 Q \sqrt{\frac{2\ln(2d+2)}{n}}.
\end{equation}
\end{proposition}

For general activation functions, we define the following norm~\footnote{Our analysis also works for the general case that $\|\bx\|_q\leq 1$ with $q\geq 1$, in which  the norm should be accordingly defined as $\sum_{k=1}^m |a_k|(\|\bb_k\|_p+|c_k|+1)$ with $q$ satisfying $1/q+1/p=1$.}. 
\begin{equation}\label{eqn: modified-path-norm-two}
    \|\theta\|_{\tP} := \sum_{k=1}^m |a_k| (\|\bb_k\|_1 + |c_k|+1).
\end{equation}
This norm is stronger than the path norm since $\|\theta\|_{\tP} = \|\theta\|_{\cP} + \sum_{k=1}^m |a_k|$.  The additional term $\sum_{k=1}^n|a_k|$ only depends on the paths of length $1$, while $\|\theta\|_{\cP}$ depends on the paths of length $2$.
The intuition behind this definition will be clear from the proof of the following theorem.

\begin{theorem}\label{Rad2-NN}
Assume that the activation function $\sigma(\cdot)$ satisfies the conditions in Theorem $\ref{1DApp}$. Let $\cF_Q = \{ f_m(\cdot;\theta) : \|\theta\|_{\tP}\leq Q, m\in \mathbb{N}_{+}\}$. Then, we have
\begin{equation}\label{eqn: rad-general}
\rad_n(\mathcal{F}_{Q})\le 2\gamma(\sigma)Q\sqrt{\frac{2\ln(2d+2)}{n}}.
\end{equation}
\end{theorem}

\begin{proof}
According to Theorem $\ref{1DApp}$, for any $\epsilon>0$, there exists a two-layer ReLU networks 
$
g_K(t;w)=\sum_{j=1}^{K}\alpha_j\sigma_R(\beta_{j} t+\gamma_j),
$
such that 
\begin{equation}\label{eqn: activation-app}
\sup_{t\in\bR}|\sigma(t)-g_K(t;w)|\le\epsilon, \,\text{and}\,  \sum_{j=1}^{K}|\alpha_j|(|\beta_{j}|+|\gamma_j|) \le\gamma(\sigma)+\epsilon
\end{equation}
For any $f(\bx)=\sum_{k=1}^ma_k\sigma(\bb_k^T\bx+c_k)\in\cF_Q$, we can decompose it  as follows:
\begin{align}\label{eqn: decomp-twolayer}
f(\bx)&=\sum_{k=1}^ma_k\left[\sigma(\bb_k^T\bx+c_k)- g_K(\bb_k^T\bx+c_k;w)\right] +\sum_{k=1}^ma_k g_K(\bb_k^T\bx+c_k;w)\\
&:=h^f_1+h^f_2,
\end{align}
where  
\begin{equation}\label{eqn: I_1}
    \sup_{\bx}\left|h^f_1(\bx)\right|\leq \sum_{k=1}^m|a_k|\epsilon\leq Q \epsilon,
\end{equation} 
and 
\begin{equation}
h^f_2(\bx;\tilde{\theta}) = \sum_{k=1}^m\sum_{j=1}^K a_k \alpha_j \sigma_R(\beta_j \bb_k^T\bx + \beta_j c_k + \gamma_j)，
\end{equation}
is a two-layer ReLU networks with the path norm 
\begin{equation}\label{eqn: I_2}
\begin{aligned}
\|\tilde{\theta}\|_{\cP} &= \sum_k\sum_j |a_k\alpha_j|(\|\beta_j\bb_k\|_1 + |\beta_jc_k+\gamma_j|)\\
&\leq \sum_k\sum_j |a_k||\alpha_j|(\|\bb_k\|_1 + |c_k|+1)(|\beta_j|+|\gamma_j|) \leq (\gamma(\sigma)+\epsilon) \|\theta\|_{\tP},
\end{aligned}
\end{equation}
where the last inequality follows \eqref{eqn: activation-app} and the definition of the norm $\|\cdot\|_{\tP}$. Therefore, $h^f_2\in \cF_{(\gamma(\sigma)+\epsilon) Q}^0$. Using the decomposition \eqref{eqn: decomp-twolayer}, we have

\begin{align}
\nonumber    \rad_n(\cF_Q) = & \frac{1}{n}\EE_{\bxi}[\sup_{f\in\cF_Q}\sum_{i=1}^n \xi_i f(\bx_i)] = \frac{1}{n}\EE_{\bxi}[\sup_{f\in\cF_Q}\sum_{i=1}^n \xi_i (h^f_1(\bx_i)+h^f_2(\bx_i))]\\
\nonumber     &\leq \frac{1}{n}\EE_{\bxi}[\sup_{f\in\cF_Q}\sum_{i=1}^n \xi_i h^f_1(\bx_i)] + \frac{1}{n}\EE_{\bxi}[\sup_{f\in\cF_Q}\sum_{i=1}^n \xi_i h^f_2(\bx_i)]\\
\nonumber     &\leq Q\epsilon + \frac{1}{n}\EE_{\bxi}[\sup_{h\in\cF_{(\gamma(\sigma)+\epsilon)Q}^0}\sum_{i=1}^n \xi_i h(\bx_i)]\\
&\leq Q\epsilon + 2(\gamma(\sigma)+\epsilon) Q\sqrt{\frac{2\ln(2d+2)}{n}},
\end{align}
where the last inequality follows from Proposition \ref{pro: two-relu-layer}. Taking $\epsilon\to 0$, we complete the proof.
\end{proof}

\begin{remark}
The reason to introduce the extra term $\sum_i |a_i|$ is clear from the inequalities \eqref{eqn: I_1} and \eqref{eqn: I_2}, which provide controls of the terms $h_1^f$ and $h_2^f$.
\end{remark}

\begin{remark}
According to Table \ref{tab: activations},  $\gamma(\sigma)=1$ for ReLU activation function. So the bound \eqref{eqn: rad-general} recovers \eqref{eqn: rad-relu}, which implies that the above theorem is tight for ReLU activation. In this case, we even have 
\[
\{f_m(\cdot;\theta): \|\theta\|_{\cP}\leq Q\} =  \{f_m(\cdot;\theta): \|\theta\|_{\tP}\leq Q\}.
\]
This follows from the fact that we can scale $\theta:=\{(a_k, \bb_k, c_k)\}$ to $\theta_t:=\{(t a_k, \bb_k/t, c_k/t)\}$ without change the function represented, but $\|\theta_t\|_{\tP} = \|\theta\|_{\cP} + t\sum_k|a_k| \to \|\theta\|_{\cP}$ as $t\to 0$.
\end{remark}

\subsection{Function space}

Let $\Omega$ be the Borel $\sigma-$algebra on $\mathbb{R}^{d+1}$ and $P(\mathbb{R}^{d+1})$ be the collection of all the probability measures on $(\mathbb{R}^{d+1},\Omega)$. Following \cite{e2018priori, e2019barron},  consider the functions that admit the following integral representation:
\begin{align}\label{fint}
    f(\bx;a,\pi)=\int_{\mathbb{R}^{d+1}} a(\bw) \sigma(\bw^T\tilde{\bx}) d\pi(\bw),\quad \forall \bx\in X,
\end{align}
where $\pi\in P(\mathbb{R}^{d+1})$, and $a(\cdot)$ is a measurable function with respect to $(\mathbb{R}^{d+1},\Omega)$. \eqref{fint} can be viewed as an infinite wide two-layer neural network.
Notice that $\sigma(\cdot)$ is a general activation function which may not enjoy the positive homogeneity property, hence the integral domain can not be normalized to the unit ball $\mathbb{S}^{d+1}$ as done in \cite{e2018priori,e2019barron}.

For any $f: X\to\bR$, we define
\begin{align}
    \|f\|_{\cB}:
    &=\inf_{(a,\pi)\in\Pi_f} \sqrt{\mathbb{E}_{\bw\sim\pi}\left[a(\bw)^2(\|\bw\|_1+1)^2\right]},\label{fBarNrmExp1}
\end{align}
where
$
    \Pi_f:=\left\{(a,\pi) : f(\bx)=\EE_{\bw\sim\pi}[ a(\bw)\sigma(\bw^T\tilde{\bx})], \forall \bx\in X\right\}.
$
 For a specific representation $(a,\pi)$, the right hand side of \eqref{fBarNrmExp1} is actually the modified path norm \eqref{eqn: modified-path-norm-two} of the infinite wide network \eqref{fint}. Notice that for a function $f$, the representations may not be unique. Therefore, it is crucial to take the infimum over $\Pi_f$, since it can make the norm independent of representations, hence becoming a function norm. 
For simplicity, we let $\cB = \{f: \|f\|_{\cB} < +\infty\}$. 

For any $\pi\in P(\bR^{d+1})$, define the kernel $k_{\pi}(\bx,\bx')=\EE_{\bw\sim\pi}[\sigma(\bw^T\tilde{\bx})\sigma(\bw^T\tilde{\bx}')]$. Let $\cH_{k_\pi}$ be the induced reproducing kernel Hilbert space (RKHS) \cite{aronszajn1950theory}. Following the work on random feature models \cite{rahimi2008random},  any $f\in \cH_{k_\pi}$ must admit the representation \eqref{fint} with $\|f\|_{\cH_{k_\pi}}=\sqrt{\EE_{\bw\sim\pi}[a(\bw)^2]}<+\infty$. With this observation, we can easily obtain 
\begin{lemma}\label{lem: rkhs-barron}
\begin{align}\label{CnnRKHS}
   \cup_{\pi\in P_c(\bR^{d+1})}  \mathcal{H}_{k_{\pi}} \subset  \cB \subset \cup_{\pi\in P(\mathbb{R}^{d+1})}\mathcal{H}_{k_{\pi}},
\end{align}
where $P_{c}(\mathbb{R}^{d+1})\subset P(\mathbb{R}^{d+1})$ is the collection of all the probability measures with compact support.
\end{lemma}

The proof can be found in Appendix \ref{sec: lemma-rkhs}.

\subsection{A priori estimates}
In this section, we  provide the a priori estimates of  the following regularized estimator:
\begin{align}
   \hat{\theta}_n =\arg\min_{\theta}J(\theta):=\hat{R}_n(\theta)+\lambda\|\theta\|_{\tilde{\cP}},
\end{align}
where $\lambda>0$ is a tuning parameter.  Notably, with the function norm defined in \eqref{fBarNrmExp1}, the analysis is almost the same as the a priori estimates of  two-layer ReLU networks \cite{e2018priori}.

In the following, we state the approximation result and the a priori estimates. The proofs can be found in Appendix \ref{sec:prf-2-NNApp} and \ref{sec: prf-apriori-two-layer}, respectively.
\begin{theorem}\label{2-NNApp}
For any $f\in \cB$  and $m\in \mathbb{N}_{+}$, there exists a two-layer neural network $f_m(\cdot;\tilde{\theta})$ with finite width $m$, such that
\begin{align}
    \mathbb{E}_{\bx}\left[f_m(\bx;\tilde{\theta})-f(\bx)\right]^2&\le \frac{3C_{\sigma}\|f\|^2_{\cB}}{m},\label{2-NNapperr}
    \\ \|\tilde{\theta}\|_{\tP}&\le 2\|f\|_{\cB},\label{2-NNpthnrm}
\end{align}
where $C_{\sigma}=(\gamma(\sigma)+\min\{|\sigma'(+\infty)|, |\sigma'(-\infty)|\}+|\sigma(0)|)^2$.
\end{theorem}
\begin{theorem}\label{thm: apriori-two-layer}
Assume that the target function $f^*\in \cB$, and $\varepsilon_i=0$.\footnote{Theorem \ref{thm: apriori-two-layer} discusses the noiseless case. In fact, the noise can be tackled under appropriate conditions, e.g. the sub-Gaussian assumption (see \cite{ma2019priori} or \cite{e2018priori}).}
Set
\begin{align}
\lambda\ge\lambda_n:=\frac{8\gamma(\sigma)\sqrt{2\ln{(2d+2)}}+1}{\sqrt{n}}.
\end{align}
Then for any $\delta\in(0,1)$, with probability at least $1-\delta$ over the random training samples $\{\bx_i\}_{i=1}^n$, the generalization error satisfies
\begin{align}\label{2-NNPriorEst}
   R(\hat{\theta}_n)\le \frac{3C_{\sigma}\|f^*\|^2_{\cB}}{2m}+2\|f^*\|_{\cB}\cdot\lambda
   +2(\|f^*\|_{\cB}+1)\cdot\lambda_n+2\sqrt{\frac{2\ln{(14/\delta)}}{n}}.
\end{align} 
Therefore, if we take $\lambda\asymp\lambda_n$, we will have
\begin{align*}
   R(\hat{\theta}_n)\lesssim \frac{\|f^*\|^2_{\cB}}{m}+\max\{1,\|f^*\|_{\cB}\}\sqrt{\frac{\ln{d}}{n}}+\sqrt{\frac{\ln{(1/\delta)}}{n}}.
\end{align*}
\end{theorem}

Notice the above estimates depend on the activation function $\sigma$ only through the appearance of constant $C_\sigma$ and the target function norm $\|f^*\|_{\cB}$. Especially, when $\sigma$ is ReLU, it exactly recover  Theorem 4.1 of \cite{e2018priori}, which is the result specifically derived for two-layer ReLU networks.

\begin{remark}
Theorem \ref{Rad2-NN} demonstrates the new path-based norm \eqref{eqn: modified-path-norm-two} is defined such that the Rademacher complexity is independent of the network width.  Theorem \ref{2-NNApp} and \ref{thm: apriori-two-layer} together show that the induced the function norm can effectively control both the approximation and estimation errors. In this sense, the complexity measure \eqref{eqn: modified-path-norm-two} for general activation function $\sigma$ is well defined. 
\end{remark}

\section{Residual neural networks}\label{SecResRad}
In this section, we consider the residual networks defined by 
\begin{align}\label{ResNet}
\begin{split}
\bh_0&=V\tilde{\bx},\\
\bh_l&=\bh_{l-1}+ U_l\sigma(W_l \bh_{l-1}),\,\, l=1,2,\cdots,L,\\
f_L(\bx;\theta)&=\balpha^T\bh_L.
\end{split}
\end{align}
Here, $\theta=\{V,W_1,U_1,W_2,U_2,\cdots,W_L,U_L,\balpha\}$ denotes the set of parameters, $V\in\bR^{D\times (d+1)}$, $W_l\in\bR^{m\times D}$, $U_l\in\bR^{D\times m}$, $\balpha\in\bR^{D}$. $L$ is the number of layers (depth), $m$ is the width of the residual blocks and $D$ is the width of skip connections. $\sigma(\cdot)$ is the (general) activation function. 

To bound the Rademacher complexity of residual networks (with general activation functions), we propose the following norms, which can be viewed as a modification of the weighted path norm defined in~\cite{ma2019priori}.
\begin{definition}\label{eqn: def-norm-resnet}
For any residual network defined as~\eqref{ResNet} with parameters $\theta$, let its norm be
\begin{align}\label{PathRes}
\|\theta\|_{\tP}=\sum_{i=0}^{L}\||\balpha|^T(I+c_\sigma|U_L||W_L|)(I+c_\sigma|U_{L-1}||W_{L-1}|)\cdots(I+c_\sigma|U_{i+1}||W_{i+1}|) |U_i|\|_1,
\end{align}
with the convention that $A_LA_{L-1}\cdots A_{i+1}=I$ when $i\ge L$. Here $U_0:=V$, and $c_\sigma>4\gamma(\sigma)+1$ is an absolute constant only related to the activation function $\sigma(\cdot)$.
\end{definition}

%{\lei Can we write down the constant $c_\sigma$, or at least an upper bound of it. This will make the reader clear how large it would be.}

With the norm defined above, we can state the theorem about the Rademacher complexity of residual networks. The bound depends linearly on the norm defined in~\eqref{PathRes}. The proof is given in Appendix~\ref{sec: pfSecResRad}. 
\begin{theorem}\label{RadResNet}
Assume that the activation function $\sigma(\cdot)$ satisfies the conditions in Theorem $\ref{1DApp}$. Let $\mathcal{F}_Q=\{f_L(x;\theta)|\|\theta\|_{\tilde{\cP}}\le Q\}$ for $Q>0$. Then we have
\begin{equation}
\rad_n(\mathcal{F}_Q)\le c_\sigma^*Q\sqrt{\frac{2\ln{(2d+2)}}{n}},
\end{equation}
where $c_\sigma^*=4\gamma(\sigma)+1$ is an absolute constant only related to the activation function $\sigma(\cdot)$.
\end{theorem}

Before the next subsection about the a priori estimate using the norm-based bounds, we make several comments on the norm~\eqref{PathRes} for residual networks. 
\begin{itemize}
\item[(1)] The norm can be viewed as an extension based on the weighted path norm proposed in~\cite{ma2019priori}, with two modifications. Firstly, the weight factor is switched from $3$ to a constant $c_\sigma$ depending on the activation function; Secondly, an additional term (modification) is added to the weighted path norm, to address the bias terms originally in the neural network or arising when approximating the activation function by ReLU networks. In fact, we have the following recursive definition (of the modified weighted path norm), which is equivalent to the norm~\eqref{PathRes}:
\begin{lemma}\label{thm: recursive_path_res}
The norm~\eqref{PathRes} can be written as a modification of the weighted path norm defined in~\cite{ma2019priori}:
\begin{align}\label{PathResr}
\|\theta\|_{\tP}&=\||\balpha|^T(I+c_{\sigma}|U_L||W_L|)\cdots(I+c_{\sigma}|U_1||W_1|)|V|\|_1\nonumber\\
&\quad+\sum_{i=1}^{L}\||\balpha|^T(I+c_\sigma|U_L||W_L|)(I+c_\sigma|U_{L-1}||W_{L-1}|)\cdots(I+c_\sigma|U_{i+1}||W_{i+1}|) |U_i|\|_1\\
&:=\|\theta\|_{\cP}+ r,\nonumber
\end{align}
where the modification term $r$ can be recursively defined as 
\begin{align}\label{PathResM}
\bM_1 &=\bm{0}_m, \nonumber \\
\bM_{l+1} &=c_{\sigma}|W_{l+1}|\sum_{k=1}^l|U_k|(\bM_k+\bm{1}_m),\quad l=1,2,\cdots,L-1, \\
r &=|\balpha|^T\sum_{l=1}^L|U_l|(\bM_l+\bm{1}_m),\nonumber
\end{align}
by selecting appropriate $\{\bM_1\}_{l=1}^L$. 
\end{lemma}

\begin{remark}
    In fact, the $\{\bM_l\}_{l=1}^L$ defined by (\ref{PathResM}) can be viewed as the modification term at the $l$-th layer. That is to say, compared to the original weighted path norm, our new norm (\ref{PathResr}) not only includes those paths starting from the biases in all the layers to the output, but also can be extended to hidden neurons, which is crucial for the proof of Theorem \ref{RadResNet}. Please see Appendix \ref{sec:ExtHidNeu} for more details.
\end{remark}

\item[(2)] We have the following upper bound for the modification terms $\bM_l$ and $r$. This theorem shows that these additional terms are not much larger than the original path norm part. 
\begin{theorem}\label{thm: path_res_control}
For $l=1,2,\cdots,L$, $i=1,2,\cdots,m$, we have
\begin{align}
M_{l,i}&\le\|c_{\sigma}|W_{l}^{i,:}|(I+|U_{l-1}|)(I+c_{\sigma}|W_{l-1}|)\cdots(I+|U_2|)(I+c_{\sigma}|W_2|)(I+|U_1|)\|_1,\label{ModControll} \\
r&\le\||\balpha|^T(I+|U_{L}|)(I+c_{\sigma}|W_{L}|)\cdots(I+|U_2|)(I+c_{\sigma}|W_2|)(I+|U_1|)\|_1,\label{ModControlL}
\end{align}
where $M_{l,i}$ is the $i$-th element of $\bM_{l}$, and $A^{i,:}$ is the $i$-th row of $A$.
\end{theorem}

\item[(3)] In the case of ReLU network without bias terms, we will not have the $+\bm{1}_m$ term in (\ref{PathResM}), the recursive definition of $\bM_l$, hence $\bM_1=\bm{0}_m$ implies $\bM_l=0$ for all $l=1,2,\cdots,L$ and $r=0$. In this case, our norm is equivalent with the weighted path norm studied in~\cite{ma2019priori}. Therefore, while our norm applies to more general cases, it works as good as the norm specifically defined for the case of ReLU activation functions. 
\end{itemize}

The proofs of Lemma \ref{thm: recursive_path_res} and Theorem \ref{thm: path_res_control} are given in Appendix~\ref{prf-recursive_path_res} and Appendix~\ref{prf-path_res_control} respectively.

\subsection{A priori estimates}
An important observation is that by splitting the two-layer neural network into several parts and stacking them vertically, we can obtain a ResNet structure like $(\ref{ResNet})$. Based on this observation and Theorem $\ref{2-NNApp}$, we can obtain the following embedding result, whose proof is given in Appendix \ref{prf-approx-resnet}.

\begin{proposition}\label{pro: approx-resnet}
For any $f\in\cB$ and $L,m\in\mathbb{N}^*$, there exists a residual network $f(x;\tilde{\theta})$ with depth $L$ and width\footnote{Here we refer to the width of residual blocks.} $m$, such that
\begin{align}
    \mathbb{E}_{\bx}\left[f(\bx;\tilde{\theta})-f(\bx)\right]^2&\le \frac{3C_{\sigma}\|f\|^2_{\cB}}{Lm},\label{apperr}
    \\ \|\tilde{\theta}\|_{\tP}&\le 2C_{\sigma,2}\|f\|_{\cB},\label{pthnrm}
\end{align}
where $C_{\sigma,2}=4\gamma(\sigma)+1$.
\end{proposition}

The above proposition implies that the norm \eqref{PathRes} is well-defined in the sense that it is compatible with the norm \eqref{eqn: modified-path-norm-two} defined for two-layer networks. 
Let us consider the regularized estimator:
\begin{align}
   \hat{\theta}_n =\arg\min_{\theta}J(\theta):=\hat{R}_n(\theta)+\lambda\|\theta\|_{\tilde{\cP}},
\end{align}
where $\lambda>0$ is a tuning parameter. Similar as the work \cite{ma2019priori}, we have 
\begin{proposition}\label{pro: apriori-resnet}
Assume $f^*\in \cB$, and choose
$
\lambda\geq\lambda_n:=((8\gamma(\sigma)+2)\sqrt{2\ln{(2d+2)}}+1)/\sqrt{n}.
$
Then for any $\delta\in(0,1)$, with probability at least $1-\delta$ over the random training samples $\{\bx_i\}_{i=1}^n$, the generalization error satisfies
\begin{align}\label{ResNetPriorEst}
   R(\hat{\theta}_n)\le \frac{3C_{\sigma}\|f^*\|^2_{\cB}}{2Lm}+2C_{\sigma,2}\|f^*\|_{\cB}\cdot\lambda
   +2(C_{\sigma,2}\|f^*\|_{\cB}+1)\cdot\lambda_n+2\sqrt{\frac{2\ln{(14/\delta)}}{n}}.
\end{align} 
Therefore, if we take $\lambda\asymp\lambda_n$, we will have
\begin{align*}
   R(\hat{\theta}_n)\lesssim \frac{\|f^*\|^2_{\cB}}{Lm}+\max\{1,\|f^*\|_{\cB}\}\sqrt{\frac{\ln{d}}{n}}+\sqrt{\frac{\ln{(1/\delta)}}{n}}.
\end{align*}
\end{proposition}

The proof of Proposition \ref{pro: apriori-resnet} can be found in Appendix \ref{prf-apriori-resnet}.

However, this proposition only implies that residual networks with proper regularization can learn functions in $\cB$, the approximation function space for two-layer neural networks. To fully understand the superiority of residual network compared to the two-layer networks, we need to identify the function spaces specialized for residual networks. 

Consider a scaled version of \eqref{ResNet}, 
\begin{equation}\label{eqn: resnet-0}
    \bh_l = \bh_{l-1} + \frac{1}{Lm} U_l \sigma(W_l \bh_{l-1}).
\end{equation}
There are two potential ways to define the function spaces. 

The first one is to  consider \eqref{eqn: resnet-0} with $L$ fixed. 
Taking $m\to \infty$ in \eqref{ResNet}, we obtain
\begin{equation}\label{eqn: resnet-2}
\begin{aligned}
    \bh_{l} &= \bh_{l-1} + \frac{1}{L}\EE_{(\bu,\bw)\sim\pi_{l}}[\bu \sigma(\bw^T\bh_{l-1})] \\
    f_L(\bx;\Pi) &= \balpha^T\bh_{L},
\end{aligned}
\end{equation}
where $\Pi=\{\pi_l\}_{l=1}^L$.
Following the expression \eqref{PathRes}, we could define the norm of  continuous ResNet \eqref{eqn: resnet-2} by
\begin{equation}\label{eqn: norm-resnet-2}
    \|\Pi\|_{\cC} := \sum_{i=0}^L |\balpha|^T(I+\frac{c_{\sigma}}{L}\EE_{\pi_L}|\bu_L||\bw^T_L|)\cdots(I+\frac{c_{\sigma}}{L}\EE_{\pi_1}|\bu_{i+1}||\bw^T_{i+1}|)\EE_{\pi_i}|\bu_i|.
\end{equation}

Another ways is to take both $L\to\infty$ and $m\to\infty$, we obtain  an ODE
\begin{equation}\label{eqn: resnet-3}
    \frac{d\bh(t)}{dt} = \EE_{(\bu,\bw)\sim\pi_t}[\bu\sigma(\bw^T\bh(t))]. 
\end{equation}
The output function is accordingly defined as $f(\bx;\pi):=\balpha^T\bh(1)$ with $\pi=(\pi_t)_{t=0}^1$. In this case, \eqref{eqn: resnet-0} can be viewed a stochastic approximation of \eqref{eqn: resnet-3}. 
\cite{e2019barron} has developed the function space theory in this way for the case $\sigma$ being ReLU, and the norm is defined through the linearized ODE: $d\bn(t)/dt = \EE_{(\bu,\bw)\sim\pi_t}[|\bu||\bw|^T] \bn(t)$.

As shown in \cite{e2019barron}, to control the Euler-forward discretization error of  \eqref{eqn: resnet-0} approximating \eqref{eqn: resnet-3}, we must make certain continuous assumption on the probability measures $(\pi_t)_{t=0}^1$, i.e. $d(\pi_t,\pi_s)\leq c|t-s|$ for some distance $d(\cdot,\cdot)$. In contrast, there is no restriction for  $\{\pi_l\}_{l=1}^L$ in \eqref{eqn: resnet-2}. 

To define function spaces appropriately in either way requires involved mathematical analysis, which we leave as the future work.

\section{Conclusion}
We propose a simple approach to build complexity measures for neural networks with general activation functions. Using this approach, we derive the path-based norms for two-layer networks and deep residual networks.
Moreover, these norms are well-defined in the sense that the Rademacher complexity can be bounded by them without depending on the network size. This property enables us to study the infinitely wide and deep networks, i.e. the function spaces induced by these norms containing functions learnable by corresponding models.

One direct future work is to complete the definition of function spaces for residual networks, which can be viewed as an extension of the function spaces defined in  \cite{e2019barron} but for deep residual networks with ReLU activation function. 

It would also be interesting to extend our results to  fully connected networks. However, this is more challenging, since we are still far away from understanding the approximation function spaces for multilayer fully connected networks, even for the three-layer case \cite{sirignano2019deep,nguyen2019mean}.

\bibliographystyle{plain}
\bibliography{ref}

\begin{thebibliography}{10}

\bibitem{aronszajn1950theory}
Nachman Aronszajn.
\newblock Theory of reproducing kernels.
\newblock {\em Transactions of the American mathematical society},
  68(3):337--404, 1950.

\bibitem{barron2018approximation}
Andrew~R Barron and Jason~M Klusowski.
\newblock Approximation and estimation for high-dimensional deep learning
  networks.
\newblock {\em arXiv preprint arXiv:1809.03090}, 2018.

\bibitem{bartlett2017nips}
Peter~L Bartlett, Dylan~J Foster, and Matus~J Telgarsky.
\newblock Spectrally-normalized margin bounds for neural networks.
\newblock In {\em Advances in Neural Information Processing Systems 30}, pages
  6240--6249, 2017.

\bibitem{Clevert2015}
Djork-Arn{\'e} Clevert, Thomas Unterthiner, and Sepp Hochreiter.
\newblock Fast and accurate deep network learning by exponential linear units
  ({ELUs}).
\newblock {\em arXiv preprint arXiv:1511.07289}, 2015.

\bibitem{ma2019priori}
Weinan E, Chao Ma, and Qingcan Wang.
\newblock A priori estimates of the population risk for residual networks.
\newblock {\em arXiv preprint arXiv:1903.02154}, 2019.

\bibitem{e2019barron}
Weinan E, Chao Ma, and Lei Wu.
\newblock Barron spaces and the compositional function spaces for neural
  network models.
\newblock {\em arXiv preprint arXiv:1906.08039}, 2019.

\bibitem{ma2019generalization}
Weinan E, Chao Ma, and Lei Wu.
\newblock On the generalization properties of minimum-norm solutions for
  over-parameterized neural network models.
\newblock {\em arXiv preprint arXiv:1912.06987}, 2019.

\bibitem{e2018priori}
Weinan E, Chao Ma, and Lei Wu.
\newblock A priori estimates of the population risk for two-layer neural
  networks.
\newblock {\em Communications in Mathematical Sciences}, 17(5):1407--1425,
  2019.

\bibitem{evans2010partial}
Lawrence~C Evans.
\newblock {\em Partial differential equations}, volume~19.
\newblock American Mathematical Soc., 2010.

\bibitem{pmlr-v75-golowich18a}
Noah Golowich, Alexander Rakhlin, and Ohad Shamir.
\newblock Size-independent sample complexity of neural networks.
\newblock In {\em Proceedings of the 31st Conference On Learning Theory}. PMLR,
  2018.

\bibitem{He2015b}
Kaiming He, Xiangyu Zhang, Shaoqing Ren, and Jian Sun.
\newblock Delving deep into rectifiers: Surpassing human-level performance on
  imagenet classification.
\newblock In {\em Proceedings of the IEEE international conference on computer
  vision}, pages 1026--1034, 2015.

\bibitem{hendrycks2016gaussian}
Dan Hendrycks and Kevin Gimpel.
\newblock Gaussian error linear units ({GELUs}).
\newblock {\em arXiv preprint arXiv:1606.08415}, 2016.

\bibitem{jason2019general}
Jason~M. Klusowski.
\newblock Total path variation for deep nets with general activation functions.
\newblock {\em
  https://jasonklusowski.github.io/files/DeepNetApproximationGeneral.pdf},
  2019.

\bibitem{lin2013NIN}
Min Lin, Qiang Chen, and Shuicheng Yan.
\newblock Network in network.
\newblock In {\em International Conference on Learning Representations}, 2014.

\bibitem{neyshabur2018a}
Behnam Neyshabur, Srinadh Bhojanapalli, and Nathan Srebro.
\newblock A {PAC}-bayesian approach to spectrally-normalized margin bounds for
  neural networks.
\newblock In {\em International Conference on Learning Representations}, 2018.

\bibitem{neyshabur2015norm}
Behnam Neyshabur, Ryota Tomioka, and Nathan Srebro.
\newblock Norm-based capacity control in neural networks.
\newblock In {\em Conference on Learning Theory}, pages 1376--1401, 2015.

\bibitem{nguyen2019mean}
Phan-Minh Nguyen.
\newblock Mean field limit of the learning dynamics of multilayer neural
  networks.
\newblock {\em arXiv preprint arXiv:1902.02880}, 2019.

\bibitem{rahimi2008random}
Ali Rahimi and Benjamin Recht.
\newblock Random features for large-scale kernel machines.
\newblock In {\em Advances in neural information processing systems}, pages
  1177--1184, 2008.

\bibitem{rahimi2008uniform}
Ali Rahimi and Benjamin Recht.
\newblock Uniform approximation of functions with random bases.
\newblock In {\em 2008 46th Annual Allerton Conference on Communication,
  Control, and Computing}, pages 555--561. IEEE, 2008.

\bibitem{ramachandran2017searching}
Prajit Ramachandran, Barret Zoph, and Quoc~V Le.
\newblock Searching for activation functions.
\newblock {\em arXiv preprint arXiv:1710.05941}, 2017.

\bibitem{savarese19a}
Pedro Savarese, Itay Evron, Daniel Soudry, and Nathan Srebro.
\newblock How do infinite width bounded norm networks look in function space?
\newblock In {\em the Thirty-Second Conference on Learning Theory}, 2019.

\bibitem{shalev2014understanding}
Shai Shalev-Shwartz and Shai Ben-David.
\newblock {\em Understanding machine learning: From theory to algorithms}.
\newblock Cambridge university press, 2014.

\bibitem{sirignano2019deep}
Justin Sirignano and Konstantinos Spiliopoulos.
\newblock Mean field analysis of deep neural networks.
\newblock {\em arXiv preprint arXiv:1903.04440}, 2019.

\bibitem{theisen2019global}
Ryan Theisen, Jason~M Klusowski, Huan Wang, Nitish~Shirish Keskar, Caiming
  Xiong, and Richard Socher.
\newblock Global capacity measures for deep {ReLU} networks via path sampling.
\newblock {\em arXiv preprint arXiv:1910.10245}, 2019.

\end{thebibliography}

\appendix

\section{Basic tools}\label{sec:BasicTools}
\subsection{Weak derivatives}\label{sec:WD}

In this section, we provide the definition of weak derivatives as well as some results that will be used in the following analysis.
We denote the classical derivatives of a one-dimensional function $u$ by $\frac{du}{dx}$, $\frac{d^2u}{dx^2}$, $\cdots$, $\frac{d^nu}{dx^n}$, and the weak derivatives by $u'$, $u''$, $\cdots$, $u^{(n)}$. 

\begin{definition}\label{def:1}
We say a function $u$ defined on $\bR$ has some properties $P$ locally, if $u$ has properties $P$ on any compact subset of $\bR$.
\end{definition}

\begin{definition}\label{def:2}
Assume that $u(x)$ is locally Lebesgue integrable on $\bR$. If there exists another locally Lebesgue integrable function $v(x)$, which satisfies the integration by parts
\begin{equation*}
\int_{\bR} u(x)\frac{d^n\phi(x)}{dx^n}dx=(-1)^n\int_{\bR} v(x)\phi(x)dx,\quad \forall \phi\in C_0^{\infty}(\bR),
\end{equation*}
then $u$ is $n$-th order weakly differentiable on $\bR$, and $v$ is the derivative, i.e. $v=u^{(n)}$. 
\end{definition}

Now we can state the following theorems required to prove Theorem \ref{1DApp}.

\begin{theorem}\label{NLFormula}
Assume that $u$ is absolutely continuous on $[a,b]$. Then $u$ is (classical) differentiable almost everywhere on $[a,b]$. Furthermore, we have $\frac{du}{dx}\in L^1[a,b]$ and
\begin{equation*}
\int_a^b \frac{du}{dx}dx=u(b)-u(a).
\end{equation*}
\end{theorem}

The proof of Theorem $\ref{NLFormula}$ can be found in any standard textbooks on real analysis.

\begin{theorem}\label{WD-AC}
Assume that $u(x)$ is locally Lebesgue integrable on $\bR$. Then $u$ is weakly differentiable $\Leftrightarrow$ $u$ is locally absolutely continuous on $\bR$.
\end{theorem}

The proof of Theorem $\ref{WD-AC}$ can be found in PDE textbooks like \cite{evans2010partial}, where a smoother technique is applied. More importantly, some intermediate results in the proof are listed below, which will be used repeatedly in the proof of Theorem~\ref{1DApp}.

\begin{corollary}\label{cly:1}
If $u$ is locally absolutely continuous on $\bR$, then $\frac{du}{dx}$ exists almost everywhere on $\bR$, and $u'(x)=\frac{du}{dx}$. 
\end{corollary}

\begin{corollary}\label{cly:2}
If $u$ is Lebesgue integrable and weakly differentiable.
%$u\in L^1(\bR)\cap D_w^1(\bR)$. 
Then for any closed interval $[a,b]$, we have 
\begin{equation}\label{NLw}
\int_a^b u'(x)dx=u(b)-u(a).
\end{equation}
\end{corollary}

We also need a proposition, which can be seen as a weak form of the mean value theorem for integrals.

\begin{proposition}\label{prop:1}
If $u\in L^1[a,b]$, then
\begin{align}\label{Mv1}
m\left(\left\{c\in [a,b]: u(c)(b-a)\le \int_{a}^b u(x)dx\right\}\right)>0,
\end{align}
and
\begin{align}\label{Mv2}
m\left(\left\{c\in [a,b]: \int_{a}^b u(x)dx \le u(c)(b-a)\right\}\right)>0,
\end{align}
where $m(E)$ denotes the measure of a measurable set $E$. Therefore, there always exists $c_1,c_2\in [a,b]$, such that
\begin{align}\label{Mv3}
u(c_1)(b-a)\le \int_{a}^b u(x)dx \le u(c_2)(b-a).
\end{align}
\end{proposition}

\begin{proof}
We only prove $(\ref{Mv1})$, since the proof of $(\ref{Mv2})$ is similar. Consider the set 
\begin{align*}
A&=\left\{c\in [a,b]: u(c)(b-a)\le \int_{a}^b u(x)dx\right\}\\
&=u^{-1}\left(\left(-\infty,\frac{1}{b-a}\int_{a}^b u(x)dx\right]\right).
\end{align*}
Obviously it is Lebesgue measurable. If $m(A)=0$, i.e. if for almost every $c\in[a,b]$, 
\begin{equation*}
u(c)>\frac{1}{b-a}\int_{a}^b u(x)dx, 
\end{equation*}
then integrating in both sides, we obtain that
\begin{equation*}
\int_{a}^b u(x)dx>\int_{a}^b u(x)dx, 
\end{equation*}
which is a contraction. 
\end{proof}

\subsection{Redemacher complexity}\label{sec:BasicRad}
The following standard results will be repeatedly used in our analysis, the proof of which can be found from classic machine learning theory books, e.g. \cite{shalev2014understanding}.

\begin{lemma}\label{RadLin}
(Linear functions) Given the samples $\{\bx_i\}_{i=1}^n\subset\mathbb{R}^d$, and the class of linear functions $\mathcal{H}=\{h(\bx)=\bu^T\bx:\|\bu\|_1\le 1\}$. Then
\begin{equation}
\rad_n(\mathcal{H})\le \max_{1\le i\le n}\|\bx_i\|_{\infty}\sqrt{\frac{2\ln(2d)}{n}}.
\end{equation}
\end{lemma}

\begin{lemma}\label{RadContraction}
(Contraction property) Assume that $\{\phi_{i}(\cdot)\}_{i=1}^n$ are Lipschitz continuous functions with a uniform Lipschitz constant $L_{\phi}$, i.e. $|\phi_{i}(x)-\phi_{i}(x')|\le L_{\phi}|x-x'|$, $x, x'\in \mathbb{R}$, $i=1,2,\cdots,n$. Then
\begin{equation}
\mathbb{E}_{\bxi}\left[\sup_{h\in\mathcal{H}}\sum_{i=1}^n\xi_i\phi_i(h(\bx_i))\right]\le L_{\phi}
\mathbb{E}_{\bxi}\left[\sup_{h\in\mathcal{H}}\sum_{i=1}^n\xi_ih(\bx_i)\right].
\end{equation}
\end{lemma}

We have the Rademacher complexity-based generalization bound as follows.
\begin{theorem}\label{RadGenGap}
Given a function class $\mathcal{H}$, for any $\delta\in(0,1)$, with probability at least $1-\delta$ over the random samples $\{\bx_i\}_{i=1}^n$,
\begin{align}
  \sup_{h\in\mathcal{H}}\left|\mathbb{E}_{\bx}[h(\bx)]-\frac{1}{n}\sum_{i=1}^nh(\bx_i)\right|\le 2\rad_n(\mathcal{H})+2\sup_{h,h'\in\mathcal{H}}\|h-h'\|_{\infty}\sqrt{\frac{2\ln{(4/\delta)}}{n}}.
\end{align} 
\end{theorem}

\section{Proofs for Section 3}

\subsection{Proofs for Theorem \ref{1DApp}}\label{sec: 1d-approx}

\subsubsection{Proof of Lemma \ref{Asymptote}}\label{sec:prf-Asymptote}
%\noindent \textbf{2. Proof of Lemma $\ref{Asymptote}$}

Based on results in Appendix \ref{sec:WD}, we can prove Lemma $\ref{Asymptote}$.

\begin{proof}
We only prove that $f$ has an asymptote when $x\rightarrow+\infty$, since the condition that $x\rightarrow-\infty$ is similar.

Firstly, we show that $\lim\limits_{x\rightarrow+\infty}f'(x)$ exists. If not, we can find a constant $\delta$ and infinite many pairs of points $\{x_n,x'_n\}_{n=1}^\infty$ that satisfy $x_n<x'_n<x_{n+1}$ for any $n\in\mathbb{N}^*$ and $|f'(x_n)-f'(x'_n)|>\delta$. By Corollary $\ref{cly:2}$, this implies
\begin{equation*}
\delta<|f'(x_n)-f'(x'_n)|=\left|\int_{x'_n}^{x_n}f''(x)dx\right|\le\int_{x'_n}^{x_n}|f''(x)|dx.
\end{equation*}
Hence, we have $\int_\bR|f''(x)|dx=+\infty$, which is contradictory with $\gamma_0(f)<+\infty$.

Then, without loss of generality, we assume $\lim\limits_{x\rightarrow+\infty}f'(x)=0$, and show that $\lim\limits_{x\rightarrow+\infty}f(x)$ exists. If not, similarly we can show that $\int_0^{+\infty}|f'(x)|dx=+\infty$. On the other hand, since $\lim\limits_{x\rightarrow+\infty}f'(x)=0$, we have
\begin{equation*}
|f'(x)|=\left|\int_x^{+\infty} f''(y)dy\right|\leq \int_x^{+\infty} |f''(y)|dy
\end{equation*}
by Corollary $\ref{cly:2}$. Therefore, according to the Fubini theorem, 
\begin{equation*}
\int_0^{+\infty} |f'(x)|dx\leq\int_0^{+\infty}\int_x^{+\infty}|f''(y)|dydx=\int_0^{+\infty} |f''(x)||x|dx\le\gamma_0(f)<+\infty.
\end{equation*}
This is contradictory with $\int_0^{+\infty}|f'(x)|dx=+\infty$. Hence, $\lim\limits_{x\rightarrow+\infty}f(x)$ exists.

Combining the existence of $f'(+\infty)$ and the existence of $f(+\infty)$ when $f'(+\infty)=0$, we can conclude that $f$ has an asymptote when $x\rightarrow+\infty$. In fact, let $f'(+\infty)=a$, then for $F(x):=f(x)-ax$, we have $\gamma_0(F)=\gamma_0(f)$ and $F'(+\infty)=0$, then $F(+\infty)$ exists. Let $F(+\infty)=b$, then we have $\lim\limits_{x\rightarrow +\infty} [f(x)-(ax+b)]=0$, $a=f'(+\infty)$ and $b=\lim\limits_{x\rightarrow +\infty} [f(x)-ax]$.
\end{proof}

\begin{remark}
The basic tools used in the proof is Cauchy criterion for convergence, Newton-Leibniz formula and Fubini theorem, where only the Newton-Leibniz formula needs to be checked in the weak form. Its validity is guaranteed by Corollary $\ref{cly:2}$.
\end{remark}

\subsubsection{Proof of Theorem \ref{1DApp}}
%Now we can get down to prove Theorem $\ref{1DApp}$.

\begin{proof}
$\forall\epsilon>0$, $\exists x_{\epsilon}\in \bR$, such that $g(x_{\epsilon})\le \inf_{x} g(x)+\epsilon/2$. Let $f_{\epsilon}(x)=f(x)-f(x_{\epsilon})-f'(x_{\epsilon})(x-x_{\epsilon})$, then we have $\gamma_0(f_\epsilon)=\gamma_0(f)\le \gamma(f)<+\infty$, $f_{\epsilon}(x_\epsilon)=f'_{\epsilon}(x_\epsilon)=0$, and $f''_{\epsilon}(x)=f''(x)$. According to Lemma $\ref{Asymptote}$, let $L^l(x)=a^lx+b^l$ be the asymptote of $f_{\epsilon}(x)$ as $x\rightarrow-\infty$, and $L^r(x)=a^rx+b^r$ be the asymptote of $f_{\epsilon}(x)$ as $x\rightarrow+\infty$. That is to say, there exists a constant $T=T(\epsilon)$ that satisfies 
\begin{align*}
|f_\epsilon(x)-L^l(x)|&<\epsilon/2,\qquad x\leq x_{\epsilon}-T,\\
|f_\epsilon(x)-L^r(x)|&<\epsilon/2,\qquad x\geq x_{\epsilon}+T.
\end{align*}
Since $\gamma_0(f)<+\infty$, we can simultaneously have 
\begin{align}
\int^{x_{\epsilon}-T}_{-\infty} |f''(x)|(|x|+1)dx+
\int_{x_{\epsilon}+T}^{+\infty} |f''(x)|(|x|+1)dx<\epsilon/32,\label{IntTail}
\end{align}
with $x_{\epsilon}-T<0<x_{\epsilon}+T$.

Consider the decomposition $f_{\epsilon}(x)=f^r_{\epsilon}(x)+f^l_{\epsilon}(x)$, where 
\begin{eqnarray*}
f^r_{\epsilon}(x)=\left\lbrace
\begin{array}{cl}
f_{\epsilon}(x),& x\ge x_{\epsilon} \\
0,& x< x_{\epsilon}
\end{array}
\right.,\quad
f^l_{\epsilon}(x)=\left\lbrace
\begin{array}{cl}
0,& x\ge x_{\epsilon} \\
f_{\epsilon}(x),& x< x_{\epsilon}
\end{array}
\right..
\end{eqnarray*}
For any positive integer $N$, make a uniform partition 
\begin{equation*}
x_{\epsilon}=x_0<x_1<\cdots<x_N=x_{\epsilon}+T,
\end{equation*}
with the step size $h=x_i-x_{i-1}$, $i=1,2,\cdots N$. Let  
\begin{equation}
\Delta_0=0;\ \Delta_i=\frac{f^r_{\epsilon}(x_i)-f^r_{\epsilon}(x_{i-1})}{x_i-x_{i-1}}, \ i=1,2,\cdots,N;\ \Delta_{N+1}=a^r,
\end{equation}
and
\begin{equation}
h_N^r(x)=\sum\limits_{i=0}^N (\Delta_{i+1}-\Delta_i)\sigma_R(x-x_i).
\end{equation}
Then $h^r_N(x)$ is a piecewise linear interpolation of $f^r_{\epsilon}(x)$ on $[x_{\epsilon},x_{\epsilon}+T]$ and $h^r_N(x)=f^r_{\epsilon}(x)=0$ on $(-\infty, x_{\epsilon})$. Hence, there exists $N_1=N_1(\epsilon)$, such that $|h^r_N(x)-f^r_{\epsilon}(x)|<\epsilon$ for any $x\in(-\infty,x_{\epsilon}+T]$ when $N\ge N_1$. For $x\ge x_{\epsilon}+T$, since $h^r_N(x)=f^r_{\epsilon}(x_N)+a^r(x-x_N)$, we have
\begin{align*}
|h^r_N(x)-f^r_{\epsilon}(x)| &\leq|h^r_N(x)-L^r(x)|+|L^r(x)-f^r_{\epsilon}(x)|\\
    &=|h^r_N(x_{\epsilon}+T)-L^r(x_{\epsilon}+T)|+|L^r(x)-f^r_{\epsilon}(x)|\\
    &\le\frac{\epsilon}{2}+\frac{\epsilon}{2}=\epsilon.
\end{align*}
Therefore, we have $|h^r_N(x)-f^r_{\epsilon}(x)|\le\epsilon$ for any $x\in \bR$.

Similarly we can construct a piecewise linear function $h^l_N(x)$ which is an $\epsilon$-approximation of $f^l_{\epsilon}(x)$ on $\bR$. Let $h_N(x)=h^l_N(x)+h^r_N(x)$, then we have $|h_N(x)-f_\epsilon(x)|<\epsilon$ for any $x\in\bR$.

On the other hand, obviously $h^r_N(x)$ is a two-layer neural network with a finite width $N+1$. Its path norm is
\begin{equation}\label{PathInt}
\|\theta(h^r_N)\|_{\cP}=\sum\limits_{i=0}^N|\Delta_{i+1}-\Delta_i|(|x_i|+1)\rightarrow\int_{x_\epsilon}^{x_\epsilon+T}|f''(x)|(|x|+1)dx,\quad N\rightarrow\infty.
\end{equation}
This limit will be verified later in Appendix \ref{sec:IntLim}. Hence, with a similar analysis for $\|\theta(h_N^l)\|_{\cP}$ and by $(\ref{IntTail})$, there exists $N_2=N_2(\epsilon)$, such that 
\begin{equation}\label{PathTol}
\|\theta(h_N)\|_{\cP}\leq\gamma_0(f)+\epsilon/2,\quad \forall N\ge N_2.
\end{equation}

Finally, let
\begin{align}
  f_N(x)=&h_N(x)+f'(x_\epsilon)(\sigma(x)-\sigma(-x))\nonumber\\
  &  +sign(f(x_\epsilon)-x_\epsilon f'(x_\epsilon))\sigma(0\cdot x+|f(x_\epsilon)-x_\epsilon f'(x_\epsilon)|),  
\end{align}
and $N\ge\max\{N_1,N_2\}=N(\epsilon)$, we have
\begin{equation*}
\|f_N(x)-f(x)\|_\infty\le\epsilon,
\end{equation*}
and
\begin{equation*}
\|\theta(f_N)\|_{\cP}\leq \|\theta(h_N)\|_{\cP}+g(x_\epsilon)\leq\gamma_0(f)+\epsilon/2+\inf_{x\in \bR} g(x)+\epsilon/2=\gamma(f)+\epsilon,
\end{equation*}
which completes the proof.
\end{proof}

\subsubsection{A supplemental proof}\label{sec:IntLim}
%\noindent\textbf{4. A supplemental proof}

This subsection gives a rigorous proof of the limit $(\ref{PathInt})$.

\begin{proof} 
The aim is to show
\begin{align*}
\lim_{N\rightarrow\infty}\|\theta(h^r_N)\|_{\cP}=\int_{x_\epsilon}^{x_\epsilon+T}|f''(x)|(|x|+1)dx.
\end{align*}
Notice that 
\begin{align*}
\|\theta(h^r_N)\|_{\cP}&=\sum\limits_{i=0}^N|\Delta_{i+1}-\Delta_i|(|x_i|+1)\\
&=\left|\frac{f_{\epsilon}(x_1)-f_{\epsilon}(x_0)}{x_1-x_0}\right|(|x_0|+1)+\left|a^r-\frac{f_{\epsilon}(x_N)-f_{\epsilon}(x_{N-1})}{x_N-x_{N-1}}\right|(|x_N|+1)\\
& \ +\sum\limits_{i=1}^{N-1}\left|\frac{f_{\epsilon}(x_{i+1})-f_{\epsilon}(x_{i})}{x_{i+1}-x_{i}}-\frac{f_{\epsilon}(x_i)-f_{\epsilon}(x_{i-1})}{x_i-x_{i-1}}\right|(|x_i|+1)
\\&=\frac{1}{h}|f_{\epsilon}(x_1)-f_{\epsilon}(x_0)|(|x_0|+1)+\left|a^r-\frac{1}{h}(f_{\epsilon}(x_N)-f_{\epsilon}(x_{N-1}))\right|(|x_N|+1)\\
& \ +\sum\limits_{i=1}^{N-1}\frac{1}{h^2}\left|{f_{\epsilon}(x_{i+1})-2f_{\epsilon}(x_{i})}+f_{\epsilon}(x_{i-1})\right|(|x_i|+1)h,
\end{align*}
we can divide the analysis into three parts.

\textbf{(1)~Boundary terms}

To obtain a Riemann Integral limit, we firstly bound the boundary term, i.e. the condition that $i=0$ and $i=N$. Note that $f''(x)$ is locally Riemann integrable, hence Lebesgue integrable on $\bR$, according to Corollary $\ref{cly:2}$ and Proposition $\ref{prop:1}$, when $i=0$, we have
\begin{align*}
\frac{1}{h}|f_{\epsilon}(x_1)-f_{\epsilon}(x_0)|(|x_0|+1)
&=\frac{1}{h}\left|\int_{x_0}^{x_1}f'_{\epsilon}(x)dx\right|(|x_0|+1)\\
&\le\frac{1}{h}\int_{x_0}^{x_1}|f'_{\epsilon}(x)|dx\cdot(|x_0|+1)\\
&\le|f'_{\epsilon}(\xi_0)|(|x_0|+1)\\
&=|f'_{\epsilon}(\xi_0)-f'_{\epsilon}(x_\epsilon)|(|x_\epsilon|+1)\\
&=\left|\int_{x_\epsilon}^{\xi_0}f''_{\epsilon}(x)dx\right|(|x_\epsilon|+1)\\
&\le|f''_{\epsilon}(\eta_0)|(\xi_0-x_\epsilon)(|x_\epsilon|+1)\\
&\le h|f''_{\epsilon}(\eta_0)|(|x_\epsilon|+1), 
\end{align*}
where $x_{\epsilon}=x_0\le\eta_0\le\xi_0\le x_1$. Let $M_\epsilon=\sup\limits_{x\in[x_\epsilon, x_\epsilon+T]}f''_{\epsilon}(x)$, then for $h=T/N<\epsilon/(16M_\epsilon(|x_\epsilon|+1))$, i.e. $N>N_{2,1}:=[16TM_\epsilon(|x_\epsilon|+1)/\epsilon]+1=N_{2,1}(\epsilon)$, we have 
\begin{align*}
\frac{1}{h}|f_{\epsilon}(x_1)-f_{\epsilon}(x_0)|(|x_0|+1)
\le \epsilon/16;
\end{align*}
when $i=N$, we similarly have
\begin{align*}
\frac{1}{h}(f_{\epsilon}(x_N)-f_{\epsilon}(x_{N-1}))&=\frac{1}{h}\int_{x_{N-1}}^{x_N}f'_{\epsilon}(x)dx \in [f'_{\epsilon}(\xi_1),f'_{\epsilon}(\xi_2)],
\end{align*}
where $\xi_1,\xi_2\in [x_{N-1},x_N]$. According to Theorem $\ref{WD-AC}$, $f'_{\epsilon}$ is absolutely continuous locally on $\bR$, certainly $f'_{\epsilon}\in C(\bR)$. Then there exists $\delta=\delta(\epsilon,x_N)=\delta(\epsilon)$, such that when $h=T/N<\delta$, i.e. $N>[T/\delta]+1:=N_{2,2}=N_{2,2}(\epsilon)$, we have $|f'_{\epsilon}(x)-f'_{\epsilon}(x_N)|<\epsilon/(32(|x_{\epsilon}+T|+1))$ for any $x$: $|x-x_N|\le \delta$. Hence, for any $N>N_{2,2}$,
%Now for any $N>N_{2,2}$, $|f'_{\epsilon}(x)-f'_{\epsilon}(x_N)|<\epsilon/(32(|x_{\epsilon}+T|+1))$ for any $x\in [x_{N-1},x_N]$. Thus 
\begin{align*}
&\left|f'_{\epsilon}(x_N)-\frac{1}{h}(f_{\epsilon}(x_N)-f_{\epsilon}(x_{N-1}))\right|(|x_{N}|+1)\\
\le & \ \max\{\left|f'_{\epsilon}(x_N)-f'_{\epsilon}(\xi_1)\right|,\left|f'_{\epsilon}(x_N)-f'_{\epsilon}(\xi_2)\right|\}(|x_{\epsilon}+T|+1)\\
< & \ \epsilon/(32(|x_{\epsilon}+T|+1))\cdot(|x_{\epsilon}+T|+1)=\epsilon/32.
\end{align*}
On the other hand, by $(\ref{IntTail})$ and Corollary $\ref{cly:2}$, we have
\begin{align*}
|f'_{\epsilon}(x_N)-a^r|(|x_{N}|+1)
&=|f'_{\epsilon}(x_N)-f'_{\epsilon}(+\infty)|(|x_{N}|+1)\\
&=\left|\int_{x_N}^{+\infty}f''_{\epsilon}(x)dx\right|(x_{N}+1)\\
&\le (x_{N}+1)\int_{x_N}^{+\infty}\left|f''_{\epsilon}(x)\right|dx\\
&\le \int_{x_N}^{+\infty}(x+1)\left|f''_{\epsilon}(x)\right|dx\\
&=\int_{x_N}^{+\infty}\left|f''(x)\right|(|x|+1)dx\\
&<\epsilon/32.
\end{align*}
Therefore 
\begin{align*}
&\left|a^r-\frac{1}{h}(f_{\epsilon}(x_N)-f_{\epsilon}(x_{N-1}))\right|(|x_N|+1)\\ \ 
\le & \ |f'_{\epsilon}(x_N)-a^r|(|x_N|+1)+\left|f'_{\epsilon}(x_N)-\frac{1}{h}(f_{\epsilon}(x_N)-f_{\epsilon}(x_{N-1}))\right|(|x_N|+1)\\
< & \ \epsilon/32+\epsilon/32=\epsilon/16.
\end{align*}

\textbf{(2)~Interior terms}

Now we can consider $i=1,2,\cdots, N-1$. Define $\delta_i=[{f_{\epsilon}(x_{i+1})-2f_{\epsilon}(x_{i})}+f_{\epsilon}(x_{i-1})]/h^2-f''(x_i)$, by Corollary $\ref{cly:2}$, we have
\begin{align*}
f_{\epsilon}(x_{i+1})-2f_{\epsilon}(x_{i})+f_{\epsilon}(x_{i-1})
&=\int_{x_i}^{x_{i+1}}f'_{\epsilon}(x)dx-\int_{x_{i-1}}^{x_{i}}f'_{\epsilon}(x)dx\\
&=\int_{x_i}^{x_{i+1}}(f'_{\epsilon}(x)-f'_{\epsilon}(x_i))dx-\int_{x_{i-1}}^{x_{i}}(f'_{\epsilon}(x)-f'_{\epsilon}(x_i))dx\\
&=\int_{x_i}^{x_{i+1}}\int_{x_i}^{x}f''(y)dy dx-\int_{x_{i-1}}^{x_{i}}\int_{x_i}^{x}f''(y)dy dx.
\end{align*}
According to the Fubini theorem, 
\begin{align*}
\delta_i&=\frac{1}{h^2}\int_{x_i}^{x_{i+1}}\int_{x_i}^{x}(f''(y)-f''(x_i))dy dx-\frac{1}{h^2}\int_{x_{i-1}}^{x_{i}}\int_{x_i}^{x}(f''(y)-f''(x_i))dy dx\\
&=\frac{1}{h^2}\int_{x_i}^{x_{i+1}}(x_{i+1}-y)(f''(y)-f''(x_i))dy -\frac{1}{h^2}\int_{x_{i-1}}^{x_{i}}(x_{i-1}-y)(f''(y)-f''(x_i))dy,
\end{align*}
therefore 
\begin{align*}
|\delta_i|&\le\frac{1}{h^2}\int_{x_i}^{x_{i+1}}(x_{i+1}-y)|f''(y)-f''(x_i)|dy + \frac{1}{h^2}\int_{x_{i-1}}^{x_{i}}(y-x_{i-1})|f''(y)-f''(x_i)|dy\\
&\le\frac{1}{h^2}w_{i+1}(f'')\int_{x_i}^{x_{i+1}}(x_{i+1}-y)dy + \frac{1}{h^2}w_i(f'')\int_{x_{i-1}}^{x_{i}}(y-x_{i-1})dy\\
&=\frac{1}{2}w_{i+1}(f'')+\frac{1}{2}w_{i}(f''),
\end{align*}
where $w_i(f)$ is the amplitude of $f$ on $[x_{i-1},x_i]$: 
\begin{align*}
w_i(f)=\sup_{x\in [x_{i-1},x_i]}f(x)-\inf_{x\in [x_{i-1},x_i]}f(x).
\end{align*}
Therefore we have
\begin{align*}
\sum_{i=1}^{N-1}\left|\delta_i\right|(|x_i|+1)h
&\le \frac{1}{2}\sum_{i=1}^{N-1}w_{i+1}(f'')(|x_i|+1)h+\frac{1}{2}\sum_{i=1}^{N-1}w_{i}(f'')(|x_i|+1)h\\
&\le \frac{1}{2}\sum_{i=1}^{N}w_{i}(f'')(|x_{i-1}|+1)h+\frac{1}{2}\sum_{i=1}^{N}w_{i}(f'')(|x_i|+1)h.
\end{align*}
Since $f''(x)$ is locally Riemann integrable, for $\epsilon'=\epsilon/(16(\max\{|x_\epsilon|,|x_\epsilon+T|\}+1))$, there exists $\delta=\delta(\epsilon)$, such that when $h=T/N<\delta$, i.e. $N>N_{2,3}:=[T/\delta]+1=N_{2,3}(\epsilon)$, we have  
\begin{align*}
\sum_{i=1}^{N}w_{i}(f'')h\le \epsilon'.
\end{align*}
Therefore 
\begin{align*}
\sum_{i=1}^{N}w_{i}(f'')(|x_i|+1)h\le \sum_{i=1}^{N}w_{i}(f'')(\max\{|x_\epsilon|,|x_\epsilon+T|\}+1)h\le\epsilon/16,
\end{align*}
which implies
\begin{align*}
\sum_{i=1}^{N-1}\left|\delta_i\right|(|x_i|+1)h \le \epsilon/16. 
\end{align*}

\textbf{(3)~The integral limit}

At last, the Riemann sum 
\begin{align*}
\sum\limits_{i=1}^{N-1}|f''(x_i)|(|x_i|+1)h\rightarrow \int_{x_\epsilon}^{x_\epsilon+T}|f''(x)|(|x|+1)dx,\quad N\rightarrow\infty,
\end{align*}
i.e. there exists $N_{2,4}=N_{2,4}(\epsilon)$, s.t. 
\begin{align*}
\left|\sum\limits_{i=1}^{N-1}|f''(x_i)|(|x_i|+1)h- \int_{x_\epsilon}^{x_\epsilon+T}|f''(x)|(|x|+1)dx\right|<\epsilon/16, \quad \forall N>N_{2,4}.
\end{align*}

\textbf{(4)~Final results}

Combining above and let $N^+_2=\max\{N_{2,1},N_{2,2},N_{2,3},N_{2,4}\}=N^+_{2}(\epsilon)$, we finally have
\begin{align*}
&\left|\sum\limits_{i=0}^N|\Delta_{i+1}-\Delta_i|(|x_i|+1)- \int_{x_\epsilon}^{x_\epsilon+T}|f''(x)|(|x|+1)dx\right|\\
\le&\ \frac{1}{h}|f_{\epsilon}(x_1)-f_{\epsilon}(x_0)|(|x_0|+1)+\left|a^r-\frac{1}{h}(f_{\epsilon}(x_N)-f_{\epsilon}(x_{N-1}))\right|(|x_N|+1)\\
&\sum\limits_{i=1}^{N-1}\left|\frac{1}{h^2}\left|{f_{\epsilon}(x_{i+1})-2f_{\epsilon}(x_{i})}+f_{\epsilon}(x_{i-1})\right|-|f''(x_i)|\right|(|x_i|+1)h\\
&+\left|\sum\limits_{i=1}^{N-1}|f''(x_i)|(|x_i|+1)h- \int_{x_\epsilon}^{x_\epsilon+T}|f''(x)|(|x|+1)dx\right|\\
\le&\ \epsilon/16+\epsilon/16+\sum\limits_{i=1}^{N-1}|\delta_i|(|x_i|+1)h+\epsilon/16\\
\le&\ \epsilon/16+\epsilon/16+\epsilon/16+\epsilon/16=\epsilon/4, \quad \forall N>N^+_{2},
\end{align*}
which completes the proof.
%With the similar bound for $\|\theta(h^l_N)\|_{\cP}$, we have that there exists $N_2=\max\{N_2^+,N_2^-\}=N_2(\epsilon)$, such that ($\ref{PathTol}$) holds for all $N>N_2$. The proof is completed.
\end{proof}

\subsection{Some examples}\label{sec: cacl-activation}

In this section, we compute (or estimate) $\gamma(f)$ for several commonly used activation functions: sigmoid, tanh, exponential ReLU, (leaky) ReLU, swish, softplus and Gaussian Error Linear Unit (GELU). The results are summarized in Table~\ref{tab: activations}.
%The definition of $\gamma(f)$ can be seen in the last section (Theorem $\ref{1DApp}$: $(\ref{gam0})$, $(\ref{g01})$ and $(\ref{gam})$).

Recall the definition of $\gamma(f)$: 
\begin{equation*}
\gamma(f) = \gamma_0(f)+\inf_{x\in \bR}g(x),
\end{equation*}
where 
\begin{align*}
\gamma_0(f) &= \int_\bR |f''(x)|(|x|+1)dx,\\
g(x)&=|f(x)|+(|x|+2)|f'(x)|.
\end{align*}
%It is easy to see $\gamma(\sigma)<\infty$ for all these activation functions, since their second derivatives are all decay exponentially (or equal zero) when $|x|\rightarrow\infty$.

\subsubsection{sigmoid}\label{sigmoid}
%\textbf{(1) sigmoid}: $\sigma_s(x)=(1+e^{-x})^{-1}$.

The sigmoid function is $\sigma_s(x)=(1+e^{-x})^{-1}$. We have $\sigma_s(x)\in C^2(\bR)$, and
\begin{align*}
\sigma'_s(x)=\frac{e^{-x}}{(1+e^{-x})^2},\ \sigma''_s(x)=\frac{e^{-x}(e^{-x}-1)}{(1+e^{-x})^3}.
\end{align*}
Then $\sigma'_s(x)>0$, and $\sigma''_s(x)>0$ when $x<0$, $\sigma''_s(x)<0$ when $x>0$. Therefore
\begin{align*}
\gamma_0(\sigma_s) &= \int_\bR |\sigma''_s(x)|(|x|+1)dx,\\
&=-\int_0^{+\infty} \sigma''_s(x)(x+1)dx+\int^0_{-\infty} \sigma''_s(x)(-x+1)dx\\
&=-\int_0^{+\infty} \sigma''_s(x)xdx-\int_0^{+\infty} \sigma''_s(x)dx-\int^0_{-\infty} \sigma''_s(x)xdx+\int^0_{-\infty} \sigma''_s(x)dx.
\end{align*}
%For any $f$ satisfies $\gamma_0(f)<\infty$, Lemma $\ref{lm:1}$ suggests that $f'(\infty)$ exists and $f(\infty)$ exists when $f'(\infty)=0$. 
Generally, by the Fubini theorem, 
\begin{align*}
\int_0^{+\infty}xf''(x)dx=\int_0^{+\infty}f''(x)\int_{0}^x 1dydx
=\int_0^{+\infty}\int_{y}^{+\infty} f''(x)dxdy=\int_0^{+\infty} (f'(+\infty)-f'(y))dy.
\end{align*}
Let $a=f'(+\infty)$, and $F(x)=f(x)-ax$, then $F'(x)=f'(x)-a$, $F'(+\infty)=0$. Therefore 
\begin{align*}
\int_0^{+\infty}xf''(x)dx&=\int_0^{+\infty} [(f'(+\infty)-a)-(f'(y)-a)]dy
=-\int_0^{+\infty} F'(y)dy=-F(+\infty)+F(0)\\
&=f(0)-\lim_{x\rightarrow+\infty}(f(x)-f'(+\infty)x).
\end{align*}
Similarly,
\begin{align*}
\int^0_{-\infty}xf''(x)dx=-f(0)+\lim_{x\rightarrow-\infty}(f(x)-f'(-\infty)x).
\end{align*}
It is easy that
\begin{align*}
\int_0^{+\infty}f''(x)dx=f'(+\infty)-f'(0),\quad \int^0_{-\infty}f''(x)dx=f'(0)-f'(-\infty).
\end{align*}
Therefore, for any $f$ has the same monotonicity as $\sigma_s$, we have 
\begin{align}\label{GammaSigmoid}
\begin{split}
\gamma_0(f) &= \lim_{x\rightarrow+\infty}(f(x)-f'(+\infty)x)-\lim_{x\rightarrow-\infty}(f(x)-f'(-\infty)x)\\
& \ -(f'(+\infty)+f'(-\infty))+2f'(0).
\end{split}
\end{align}
Take $f=\sigma_s$, use the fact that $\sigma'_s(+\infty)=0$, $\sigma'_s(-\infty)=0$, $\sigma_s(+\infty)=1$, $\sigma_s(-\infty)=0$, $\sigma_s(0)=1/2$, $\sigma'_s(0)=1/4$, we have $\gamma_0(\sigma_s)=3/2$. 

Denote $g(x)$ for sigmoid by $g_s(x)$. Notice that $g_s(x)\ge 0$ and $g_s(-\infty)=0$, we have $\inf_x g_s(x)=0$, therefore $\gamma(\sigma_s)=\gamma_0(\sigma_s)=3/2$.
\begin{comment}
$g_s(x)=\frac{1}{1+e^{-x}}+\frac{2e^{-x}}{(1+e^{-x})^2}$. Let $t=e^{-x}>0$, using chain rule to have 
\begin{align*}
g'_s(x)=\frac{(3e^{-x}-1)e^{-x}}{(1+e^{-x})^3},
\end{align*}
but $g'_s(x)=0$ leads to the only global maximum. Therefore $\inf_xg_s(x)=\min\{g_s(-\infty),g_s(+\infty)\}=g_s(-\infty)=0$.
\end{comment}

\subsubsection{tanh}\label{tanh}
%\textbf{(2) tanh}: $\sigma_t(x)=\frac{e^x-e^{-x}}{e^x+e^{-x}}$.

The tanh function is $\sigma_t(x)=\frac{e^x-e^{-x}}{e^x+e^{-x}}$. We have $\sigma_t(x)\in C^2(\bR)$, and
\begin{align*}
\sigma'_t(x)=\frac{4}{(e^x+e^{-x})^2},\ \sigma''_t(x)=\frac{8e^{-x}(1-e^{2x})}{(e^x+e^{-x})^3}.
\end{align*}
Then $\sigma'_t(x)>0$, and $\sigma''_t(x)>0$ when $x<0$, $\sigma''_t(x)<0$ when $x>0$, just like the monotonicity of $\sigma_s(\cdot)$. Take $f=\sigma_t$ in ($\ref{GammaSigmoid}$), and notice that $\sigma'_t(+\infty)=0$, $\sigma'_t(-\infty)=0$, $\sigma_t(+\infty)=1$, $\sigma_t(-\infty)=-1$, $\sigma_t(0)=0$, $\sigma'_t(0)=1$, we have $\gamma_0(\sigma_t)=4$. 

Denote $g(x)$ for tanh by $g_t(x)$. Then
\begin{align*}
g_t(x)=\frac{|e^x-e^{-x}|}{e^x+e^{-x}}+\frac{4(|x|+2)}{(e^x+e^{-x})^2}
\ge\frac{|e^x-e^{-x}|}{e^x+e^{-x}}+\frac{8}{(e^x+e^{-x})^2}:=\tilde{g}_t(x),
\end{align*}
i.e.
\begin{align*}
\tilde{g}_t(x)=\left\lbrace
\begin{array}{cl}
\frac{e^x-e^{-x}}{e^x+e^{-x}}+\frac{8}{(e^x+e^{-x})^2},& x\ge 0 \\
\frac{-e^x+e^{-x}}{e^x+e^{-x}}+\frac{8}{(e^x+e^{-x})^2},& x<0 
\end{array}
\right..
\end{align*}
Let $t=e^{2x}>0$, and using the chain rule to have
\begin{align*}
\tilde{g}'_t(x)=\left\lbrace
\begin{array}{cl}
\frac{4e^{2x}(5-3e^{2x})}{(e^{2x}+1)^3},& x>0 \\
\frac{4e^{2x}(3-5e^{2x})}{(e^{2x}+1)^3},& x<0 
\end{array}
\right.,
\end{align*}
but $\tilde{g}'_t(x)=0$ leads to two local maximums. Therefore, $\inf_x \tilde{g}_t(x)=\min\{\tilde{g}_t(-\infty),\tilde{g}_t(0),\tilde{g}_t(+\infty)\}=\tilde{g}_t(-\infty)=\tilde{g}_t(+\infty)=1$. Notice that ${g}_t(-\infty)=1$, we have $\inf_x g_t(x)={g}_t(-\infty)=1$. Combining all above gives $\gamma(\sigma_t)=5$.

\subsubsection{exponential ReLU}\label{eReLU}
%\textbf{(3) exponential ReLU}: 

The exponential ReLU function is
$\sigma_e(x)=\left\lbrace
\begin{array}{cl}
x,& x\ge 0 \\
\alpha(e^x-1),& x<0 
\end{array}
\right.
$.
When $\alpha=1$, $\sigma_e(x)$ is twice weakly differentiable on $\bR$, and
$
\sigma'_{e}(x)=\left\lbrace
\begin{array}{cl}
1,& x\ge 0 \\
e^x,& x<0 
\end{array}
\right.
$, $
\sigma''_{e}(x)=\left\lbrace
\begin{array}{cl}
0,& x\ge 0 \\
e^x,& x<0 
\end{array}
\right.
$.
Then
\begin{align*}
\gamma_0(\sigma_e) = \int_\bR |\sigma''_{e}(x)|(|x|+1)dx
=\int^0_{-\infty} e^x(-x+1)dx=2.
\end{align*}

Denote $g(x)$ for exponential ReLU by $g_e(x)$. Then
$g_e(x)=\left\lbrace
\begin{array}{cl}
2(x+1),& x\ge 0 \\
(1-x)e^x+1,& x<0 
\end{array}
\right.
$,
thus $\inf_xg_e(x)=g_e(-\infty)=1$. Combining above gives $\gamma(\sigma_e)=3$.

When $\alpha\neq 1$, $\sigma_e(x)$ is not twice weakly differentiable on $\bR$.
%$\sigma_e(x)\notin D_w^2(\bR)$. 
Notice that $\sigma_e$ is continuous and piecewise smooth, according to the proof of Theorem $\ref{1DApp}$, we are supposed to define \footnote{The reason is that the two-layer ReLU network for approximation is constructed on $(-\infty,x_{0}]$ and $[x_{0},+\infty)$ respectively in the proof, where $x_0$ can be pre-selected on demand.}
\begin{equation*}
\tilde{\gamma}(f) = \tilde{\gamma}_0(f)+|f(x_0)|+(1+|x_0|)(|f'_+(x_0)|+|f'_-(x_0)|),
\end{equation*}
where 
\begin{align*}
\tilde{\gamma}_0(f) = \int_{x_0}^{+\infty}|f''(x)|(|x|+1)dx+\int^{x_0}_{-\infty}|f''(x)|(|x|+1)dx,
\end{align*}
and $x_0$ is the only ``singular'' point (here $x_0=0$). Then it is easy to have
\begin{align*}
\tilde{\gamma}_0(\sigma_e) &= \int^0_{-\infty}|\alpha|e^x(-x+1)dx=2|\alpha|, \\
\tilde{\gamma}(\sigma_e) &= 3|\alpha|+1.
\end{align*}

\subsubsection{(leaky) ReLU}\label{lReLU}
%\textbf{(4) (leaky) ReLU}: 

The (leaky) ReLU function is $\sigma_R(x)=\max(\lambda x,x)$, where $\lambda \in [0,1)$. Obviously $\sigma_R(x)$ is continuous and piecewise smooth, but not twice weakly differentiable. We have
$\sigma'_R(x)=\left\lbrace
\begin{array}{cl}
1,& x>0 \\
\lambda,& x<0 
\end{array}
\right.
$, and 
$\sigma''_R(x)=\left\lbrace
\begin{array}{cl}
0,& x>0 \\
0,& x<0 
\end{array}
\right.
$. Thus $\tilde{\gamma}_0(\sigma_R)=0$, and $\tilde{\gamma}(\sigma_R)=\lambda+1$.

\subsubsection{swish}\label{swish}
%\textbf{(5) swish}: 

The swish function is $f(x)=x\sigma_s(\beta x)$, where $\sigma_s(x)$ is the sigmoid function $\sigma_s(x)=(1+e^{-x})^{-1}$, and $\beta>0$. We have $f(x)\in C^2(\bR)$, and
\begin{align*}
f'(x)&=\sigma_s(\beta x)+\beta x\sigma_s'(\beta x)=\frac{1+(1+\beta x)e^{-\beta x}}{(1+e^{-\beta x})^2},\\
f''(x)&=2\beta\sigma_s'(\beta x)+\beta^2 x\sigma_s''(\beta x)=\frac{\beta e^{-\beta x}}{(1+e^{-\beta x})^3}[e^{-\beta x}(\beta x+2)-(\beta x-2)].
\end{align*}
Now we need to determine the sign of $f''(x)$. Let $f''(x)=0$, we get 
\begin{equation}\label{f2Eq}
e^{-\beta x}=1-\frac{4}{\beta x+2}.
\end{equation}
It is easy to see that ($\ref{f2Eq}$) has two different roots $x_1<-2/\beta <0<x_2$, where $x_1:=x_{1}(\beta)$ and $x_2:=x_{2}(\beta)$. Furthermore, we have $f''(x)<0$ when $x<x_1$, and $f''(x)>0$ when $x\in (x_1,x_2)$, and $f''(x)<0$ when $x>x_2$. Therefore, 
\begin{align*}
\gamma_0(f) &= \int_\bR |f''(x)|(|x|+1)dx\\
&=\int^{x_1}_{-\infty} (-f''(x))(-x+1)dx+\int_{x_1}^{0} f''(x)(-x+1)dx\\
& \quad +
\int^{x_2}_{0} f''(x)(x+1)dx+
\int_{x_2}^{+\infty} (-f''(x))(x+1)dx\\
&=\int^{x_1}_{-\infty} f''(x)xdx-\int^{x_1}_{-\infty} f''(x)dx-\int_{x_1}^{0} f''(x)xdx+\int_{x_1}^{0} f''(x)dx\\
& \quad +
\int^{x_2}_{0} f''(x)xdx+\int^{x_2}_{0} f''(x)dx-
\int_{x_2}^{+\infty} f''(x)xdx-
\int_{x_2}^{+\infty} f''(x)dx\\
&=\int^{0}_{-\infty} f''(x)xdx-2\int_{x_1}^{0} f''(x)xdx-\int_{0}^{+\infty} f''(x)xdx+2\int^{x_2}_{0} f''(x)xdx\\
& \quad -(f'(x_1)-f'(-\infty))+(f'(0)-f'(x_1))+(f'(x_2)-f'(0))-(f'(+\infty)-f'(x_2)).
\end{align*}
Similar to the computation in Appendix \ref{sigmoid}, we have
\begin{align*}
\int_0^{+\infty}xf''(x)dx&=f(0)-\lim_{x\rightarrow+\infty}(f(x)-f'(+\infty)x),\\
\int_0^b xf''(x)dx&=f(0)-(f(b)-f'(b)b),\quad \forall b>0,
\end{align*}
and 
\begin{align*}
\int^0_{-\infty}xf''(x)dx&=-f(0)+\lim_{x\rightarrow-\infty}(f(x)-f'(-\infty)x),\\
\int^0_b xf''(x)dx&=-f(0)+(f(b)-f'(b)b),\quad \forall b<0.
\end{align*}
Combining above gives 
\begin{align*}
\gamma_0(f) &= \left\{\lim_{x\rightarrow+\infty}(f(x)-f'(+\infty)x)+\lim_{x\rightarrow-\infty}(f(x)-f'(-\infty)x)-(f'(+\infty)-f'(-\infty))+2f(0)\right\}\\
& \ +\left\{2(x_1f'(x_1)+x_2f'(x_2))-2(f'(x_1)-f'(x_2))-2(f(x_1)+f(x_2))\right\}:=C_1(\beta)+C_2(\beta).
\end{align*}
We have $f'(+\infty)=1$, $f'(-\infty)=0$, $f(0)=0$, $f(-\infty)=0$, %$\lim\limits_{x\rightarrow +\infty}(f(x)-x)=\lim\limits_{x\rightarrow +\infty}\frac{-x}{e^{\beta x}+1}=\lim\limits_{x\rightarrow +\infty}\frac{-1}{\beta e^{\beta x}}=0$, 
$\lim\limits_{x\rightarrow +\infty}(f(x)-x)=0$, so $C_1(\beta)=-1$. Let $t_i=\beta x_i$, $i=1,2$, then $t_1<-2<0<t_2$ are two different roots of the equation $e^{-t}=1-\frac{4}{t+2}$. Therefore 
\begin{align*}
C_2(\beta)&=2(\beta x_1^2\sigma_s'(\beta x_1)+\beta x_2^2\sigma_s'(\beta x_2)-\sigma_s(\beta x_1)-\beta x_1\sigma_s'(\beta x_1)+\sigma_s(\beta x_2)+\beta x_2\sigma_s'(\beta x_2))\\
&=\frac{2}{\beta}(t_1^2\sigma_s'(t_1)+t_2^2\sigma_s'(t_2))+2[(t_2\sigma_s'(t_2)-t_1\sigma_s'(t_1))+(\sigma_s(t_2)-\sigma_s(t_1))]:=\frac{1}{\beta}c_1+c_2,
\end{align*}
where $c_1=2(t_1^2\sigma_s'(t_1)+t_2^2\sigma_s'(t_2))$ and $c_2=2[(t_2\sigma_s'(t_2)-t_1\sigma_s'(t_1))+(\sigma_s(t_2)-\sigma_s(t_1))]$ are two constants not related to $\beta$. As a result, $\gamma_0(f)=\frac{1}{\beta}c_1+c_2-1$.

On the other hand, by L'Hospital rule,
\begin{align*}
\lim_{x\rightarrow-\infty}\sigma_s(x)=\lim_{x\rightarrow-\infty}x\sigma_s(x)=\lim_{x\rightarrow-\infty}x\sigma'_s(x)=\lim_{x\rightarrow-\infty}x^2\sigma'_s(x)=0,
%\lim_{x\rightarrow-\infty}\sigma_s(x)&=\lim_{x\rightarrow-\infty}\frac{1}{1+e^{-x}}=0,\\
%\lim_{x\rightarrow-\infty}x\sigma_s(x)&=\lim_{x\rightarrow-\infty}\frac{x}{1+e^{-x}}=\lim_{x\rightarrow-\infty}\frac{1}{-e^{-x}}=0,\\
%\lim_{x\rightarrow-\infty}x\sigma'_s(x)&=\lim_{x\rightarrow-\infty}\frac{xe^{-x}}{(1+e^{-x})^2}=\lim_{x\rightarrow-\infty}\frac{x}{e^x+e^{-x}+2}=\lim_{x\rightarrow-\infty}\frac{1}{e^x-e^{-x}}=0,\\
%\lim_{x\rightarrow-\infty}x^2\sigma'_s(x)&=\lim_{x\rightarrow-\infty}\frac{x^2}{e^x+e^{-x}+2}=\lim_{x\rightarrow-\infty}\frac{2x}{e^x-e^{-x}}=\lim_{x\rightarrow-\infty}\frac{2}{e^x+e^{-x}}=0,
\end{align*}
and notice that  
\begin{align*}
g(x)&=|f(x)|+(|x|+2)|f'(x)|=| x\sigma_s(\beta x)|+(|x|+2)|\sigma_s(\beta x)+\beta x\sigma_s'(\beta x)|\\
&\le \frac{2}{\beta}|\beta x\sigma_s(\beta x)|+2|\sigma_s(\beta x)|+\frac{1}{\beta}|(\beta x)^2\sigma_s'(\beta x)|+2|\beta x\sigma_s'(\beta x)|,
\end{align*}
we have $\inf_{x}g(x)=g(-\infty)=0$. Combining all above gives $\gamma(f)=\gamma_0(f)=\frac{1}{\beta}c_1+c_2-1$. \footnote{Notice that $e^{-t}=1-\frac{4}{t+2}=\frac{t-2}{t+2}$ implies $e^{t}=\frac{t+2}{t-2}=1-\frac{4}{2-t}$, we just need to solve $e^{-t}=1-\frac{4}{t+2}$ for $t\ge 0$. A numerical result can be given: $t_2=-t_1\approx 2.3994$, $c_1\approx 1.7569$, $c_2\approx 2.3994$, thus $\gamma(f)\approx\frac{1.7569}{\beta}+1.3994$.}

\subsubsection{softplus}\label{softplus}
%\textbf{(6) softplus}: $f(x)=\ln{(1+e^x)}$.

The softplus function is $f(x)=\ln{(1+e^x)}$. We have $f(x)\in C^2(\bR)$, and
\begin{align*}
f'(x)=\frac{e^x}{1+e^{x}}>0,\
f''(x)=\frac{e^x}{(1+e^{x})^2}>0.
\end{align*}
Therefore 
\begin{align*}
\gamma_0(f) &= \int_\bR |f''(x)|(|x|+1)dx\\
&=\int^{0}_{-\infty} f''(x)(-x+1)dx+
\int_{0}^{+\infty} (f''(x))(x+1)dx\\
&=-\int^{0}_{-\infty} f''(x)xdx+\int^{0}_{-\infty} f''(x)dx
+\int_{0}^{+\infty} f''(x)xdx+
\int_{0}^{+\infty} f''(x)dx\\
&=-\lim_{x\rightarrow+\infty}(f(x)-f'(+\infty)x)-\lim_{x\rightarrow-\infty}(f(x)-f'(-\infty)x)+(f'(+\infty)-f'(-\infty))+2f(0).
\end{align*}
We have $f'(+\infty)=1$, $f'(-\infty)=0$, $f(0)=\ln{2}$, $f(-\infty)=0$,
%$\lim\limits_{x\rightarrow +\infty}(f(x)-x)=\lim\limits_{x\rightarrow +\infty}\ln{(1+e^{-x})}=0$,
$\lim\limits_{x\rightarrow +\infty}(f(x)-x)=0$, so $\gamma_0(f)=1+2\ln{2}$. Meanwhile, it is easy to check
$g(-\infty)=0$, therefore $\gamma(f)=\gamma_0(f)=1+2\ln{2}\approx 2.3863$. 

\subsubsection{Gaussian Error Linear Unit (GELU)}\label{GELU}
%\textbf{(7) Gaussian Error Linear Unit (GELU)}: 
The Gaussian Error Linear Unit (GELU) function is $f(x)=x\Phi(x)$, where $\Phi(x)$ is the cumulative distribution function of the standard normal distribution
\begin{align*}
\Phi(x)=\int_{-\infty}^x\frac{1}{\sqrt{2\pi}}e^{-\frac{t^2}{2}}dt:=\int_{-\infty}^x \phi(t)dt.
\end{align*}
We have $f(x)\in C^2(\bR)$, and
\begin{align*}
f'(x)&=\Phi(x)+x\phi(x),\\
f''(x)&=2\phi(x)+x\phi'(x)=\phi(x)(2-x^2).
\end{align*}
Since $\phi(x)>0$, we have that $f''(x)<0$ when $x<x_1$, and $f''(x)>0$ when $x\in (x_1,x_2)$, and $f''(x)<0$ when $x>x_2$, where $x_1=-\sqrt{2}$ and $x_2=\sqrt{2}$. The monotonicity is the same as swish except the different critical points $x_{1}, x_{2}$. According to the expression of $\gamma_0(f)$ in Appendix \ref{swish}, we have
\begin{align*}
\gamma_0(f) &= \left\{\lim_{x\rightarrow+\infty}(f(x)-f'(+\infty)x)+\lim_{x\rightarrow-\infty}(f(x)-f'(-\infty)x)-(f'(+\infty)-f'(-\infty))+2f(0)\right\}\\
& \ +\left\{2(x_1f'(x_1)+x_2f'(x_2))-2(f'(x_1)-f'(x_2))-2(f(x_1)+f(x_2))\right\}:=C+C_{1,2}.
\end{align*}
We have $f(0)=0$, and by L'Hospital rule,
\begin{align*}
f'(+\infty)=1, f'(-\infty)=0, f(-\infty)=0. 
%f'(+\infty)&=\Phi(+\infty)+\lim_{x\rightarrow +\infty}x\cdot\frac{1}{\sqrt{2\pi}}e^{-\frac{x^2}{2}}=1+0=1,\\
%f'(-\infty)&=\Phi(-\infty)+\lim_{x\rightarrow -\infty}x\cdot\frac{1}{\sqrt{2\pi}}e^{-\frac{x^2}{2}}=0+0=0,\\
%f(-\infty)&=\lim_{x\rightarrow -\infty}\frac{1}{x}\int_{-\infty}^x\frac{1}{\sqrt{2\pi}}e^{-\frac{t^2}{2}}dt=\lim_{x\rightarrow -\infty}\frac{\frac{1}{\sqrt{2\pi}}e^{-\frac{x^2}{2}}}{-\frac{1}{x^2}}=0.
\end{align*}
For any $x>0$, notice that
\begin{align*}
|f(x)-x|=x(1-\Phi(x))=
\frac{1}{\sqrt{2\pi}}\int^{+\infty}_x xe^{-\frac{t^2}{2}}dt
\le \frac{1}{\sqrt{2\pi}}\int^{+\infty}_x te^{-\frac{t^2}{2}}dt=\frac{1}{\sqrt{2\pi}}e^{-\frac{x^2}{2}},
\end{align*}
we have $\lim\limits_{x\rightarrow +\infty}(f(x)-f'(+\infty)x)=0$, thus $C=-1$. 

On the other hand, using the fact that $x_1=-x_2$ and $\Phi(x)-\Phi(-x)=2\Phi(x)-1$ for any $x>0$, we have 
\begin{align*}
C_{1,2}&=2[\Phi(x_2)-\Phi(x_1)+x_1\phi(x_1)(x_1-1)+x_2\phi(x_2)(x_2+1)]
\\&=2[2\Phi(x_2)-1+2x_2\phi(x_2)(x_2+1)]=4\left(\Phi(\sqrt{2})+\frac{1+\sqrt{2}}{e\sqrt{\pi}}\right)-2.
\end{align*}

In the meanwhile, it is not hard to have
$g(-\infty)=0$, therefore $\gamma(f)=\gamma_0(f)=4\left(\Phi(\sqrt{2})+\frac{1+\sqrt{2}}{e\sqrt{\pi}}\right)-3\approx 2.6897$.

\section{Proofs for Section \ref{sec: two-layer}} \label{app: two-layer}

\subsection{Proof of Lemma \ref{lem: rkhs-barron}} \label{sec: lemma-rkhs}

\begin{proof}
For any $\pi\in P(\mathbb{R}^{d+1})$, according to \cite{rahimi2008uniform}, 
\begin{align}\label{RKHS}
    \mathcal{H}_{k_{\pi}}=\left\{\int_{\mathbb{R}^{d+1}} a(\bw)\sigma(\bw^T\tilde{\bx}) d\pi(\bw):\mathbb{E}_{\bw\sim\pi} \left[a(\bw)^2\right]<+\infty\right\},
\end{align}
moreover $\|f\|^2_{\cH_{k_\pi}} = \EE_{\bw\sim\pi}[a(\bw)^2]$. For any $f\in\cB$, $f$ also has the integral representation ($\ref{fint}$)
\begin{align*}
    f(\bx)=\int_{\mathbb{R}^{d+1}} a(\bw)\sigma(\bw^T\tilde{\bx}) d\pi(\bw),\quad \forall \bx\in X,
\end{align*}
and
\begin{align*}
    \|f\|_{\cB}^2=\inf_{(a,\pi)\in\Pi_f}
    \mathbb{E}_{\bw\sim\pi}\left[a(\bw)^2(\|\bw\|_1+1)^2\right]<+\infty
\end{align*}
by (\ref{fBarNrmExp1}). Then there exists $(a_0,\pi_0)\in\Pi_f$, such that
\begin{align*}
    \mathbb{E}_{\bw\sim\pi_0}\left[a_0(\bw)^2(\|\bw\|_1+1)^2\right]
    \le \|f\|_{\cB}^2+1<+\infty.
\end{align*}
Notice that 
\begin{align*}
  \mathbb{E}_{\bw\sim\pi_0}\left[a_0(\bw)^2(\|\bw\|_1+1)^2\right]\ge \mathbb{E}_{\bw\sim\pi_0}\left[a_0(\bw)^2\right], 
\end{align*}
we have $f\in\mathcal{H}_{k_{\pi_0}}$.

On the other hand, for any $\pi\in P_c(\mathbb{R}^{d+1})$, denote the support of $\pi$ by $supp(\pi)$, then $supp(\pi)\subset\mathbb{R}^{d+1}$ is compact, hence it is bounded and closed. Let
\begin{align*}
    M_{\pi}:=\sup_{\bw\in supp(\pi)}(\|\bw\|_1+1)^2,
\end{align*}
then $M_{\pi}<+\infty$. Therefore
\begin{align*}
    \int_{\mathbb{R}^{d+1}}a(\bw)^2(\|\bw\|_1+1)^2 d\pi(\bw)
    &=\int_{supp(\pi)}|a(\bw)|^2(\|\bw\|_1+1)^2 d\pi(\bw)\\
    &\le M_{\pi} \int_{supp(\pi)}|a(\bw)|^2 d\pi(\bw)\\
    &=M_{\pi}\int_{\mathbb{R}^{d+1}}|a(\bw)|^2 d\pi(\bw).
\end{align*}
For any $f\in\mathcal{H}_{k_{\pi}}$, where $\pi\in P_c(\mathbb{R}^{d+1})$, we have
\begin{align*}
    f(\bx)=\int_{\mathbb{R}^{d+1}} a(\bw)\sigma(\bw^T\tilde{\bx}) d\pi(\bw),\quad \int_{\mathbb{R}^{d+1}} a(\bw)^2 d\pi(\bw)<+\infty
\end{align*}
by $(\ref{RKHS})$. Combining above gives
\begin{align*}
    \|f\|_{\cB}^2\le\mathbb{E}_{\bw\sim\pi}\left[a(\bw)^2(\|\bw\|_1+1)^2\right]\le M_{\pi}\mathbb{E}_{\pi}\left[a(\bw)^2\right]<+\infty,
\end{align*}
which completes the proof.
\end{proof}

\subsection{Proof of Theorem \ref{2-NNApp}}\label{sec:prf-2-NNApp}

\begin{proof}
By definition, for any $f\in\cB, \epsilon>0$, there exists $(a,\pi):=(a_\epsilon,\pi_\epsilon)\in\Pi_f$, such that
\begin{align}\label{BarInf1}
    \mathbb{E}_{\bw\sim\pi}\left[a(\bw)^2(\|\bw\|_1+1)^2\right]\le \|f\|^2_{\cB}+\epsilon.
\end{align}
Write $(\ref{fint})$ to an expectation form
\begin{align}\label{fExp}
    f(\bx)=\mathbb{E}_{\bw\sim\pi}\left[a(\bw)\sigma(\bw^T\tilde{\bx})\right],
\end{align}
then it is natural to sample i.i.d. random variables $U=\{\bw_i\}_{i=1}^m$ from the distribution $\pi(\cdot)$, and define a two-layer neural network 
\begin{align*}
    \hat{f}_U(\bx)=\frac{1}{m}\sum_{i=1}^ma(\bw_i)\sigma(\bw_i^T\tilde{\bx}).
\end{align*}
Let $L_U=\mathbb{E}_{\bx}|\hat{f}_U(\bx)-f(\bx)|^2$ denote the population risk, then we have
\begin{align*}
   \mathbb{E}_U[L_U]&=\mathbb{E}_{\bx}\mathbb{E}_U|\hat{f}_U(\bx)-f(\bx)|^2\\
   &=\frac{1}{m^2}\mathbb{E}_{\bx}\left[\sum_{i\ne j}\mathbb{E}_{\bw_i,\bw_j}\big[(a(\bw_i)\sigma(\bw_i^T\tilde{\bx})-f(\bx))(a(\bw_j)\sigma(\bw_j^T\tilde{\bx})-f(\bx))\big]\right]\\
   & \quad +\frac{1}{m^2}\mathbb{E}_{\bx}\left[\sum_{i=1}^m\mathbb{E}_{\bw_i}\big[(a(\bw_i)\sigma(\bw_i^T\tilde{\bx})-f(\bx))^2\big]\right],
\end{align*}
where $\mathbb{E}_{\bw}:=\mathbb{E}_{\bw\sim\pi}$. Notice that $\{\bw_i\}_{i=1}^m$ are i.i.d samples. So together with \eqref{fExp}, we have  for any $i,j\in [m]$ and $i\neq j$,
\begin{align*}
   &\mathbb{E}_{\bw_i,\bw_j}\big[(a(\bw_i)\sigma(\bw_i^T\tilde{\bx})-f(\bx))(a(\bw_j)\sigma(\bw_j^T\tilde{\bx})-f(\bx))\big]
   =0.
\end{align*}
Let $L_\sigma$ denote the Lipschitz constant of $\sigma$, then it is easy to have 
\begin{align}\label{ActLin}
   |\sigma(y)|\le|\sigma(y)-\sigma(0)|+|\sigma(0)|\le L_{\sigma}|y|+|\sigma(0)|.
\end{align}
By $(\ref{fExp})$, $(\ref{ActLin})$ and H$\rm\ddot{o}$lder's inequality, we have
\begin{align}
    &\mathbb{E}_{\bw}\left[(a(\bw)\sigma(\bw^T\tilde{\bx})-f(\bx))^2\right] \nonumber\\
   \le~&\mathbb{E}_{\bw}\left[a(\bw)^2\sigma(\bw^T\tilde{\bx})^2\right] \label{2ndMoment}\\
   \le~&\mathbb{E}_{\bw}\left[a(\bw)^2(L_{\sigma}\|\bw\|_1\|\tilde{\bx}\|_{\infty}+|\sigma(0)|)^2\right] \nonumber\\
   \le~& (L_{\sigma}+|\sigma(0)|)^2~\mathbb{E}_{\bw}\left[a(\bw)^2(\|\bw\|_1+1)^2\right] \nonumber\\
   \le~& C_{\sigma}(\|f\|^2_{\cB}+\epsilon), \nonumber
\end{align}
where the last inequality is due to $(\ref{BarInf1})$. Combining above gives
\begin{align*}
   \mathbb{E}_U[L_U]\le \frac{C_{\sigma}(\|f\|^2_{\cB}+\epsilon)}{m}.
\end{align*}

Denote the modified path norm of $\hat{f}_U(x)$ by $A_U$, i.e. 
$
   A_U=\frac{1}{m}\sum_{i=1}^m |a(\bw_i)|(\|\bw_i\|_1+1),
$
then it is easy to have 
\begin{align*}
   \mathbb{E}_U[A_U]=\frac{1}{m}\sum_{i=1}^m \mathbb{E}_{\bw_i}\big[|a(\bw_i)|(\|\bw_i\|_1+1)\big]
   =\mathbb{E}_{\bw}\big[|a(\bw)|(\|\bw\|_1+1)\big],
   %&\le \|f\|^2_{\tilde{B}}+\epsilon,
\end{align*}
which implies $\mathbb{E}_U^2[A_U]\le \|f\|^2_{\cB}+\epsilon$ due to Jensen's inequality and $(\ref{BarInf1})$. Now we obtain
\begin{align*}
   \mathbb{E}_U[A_U]\le \|f\|_{\cB}+\sqrt{\epsilon}.
\end{align*}

Consider the event $E_1:=\left\{L_U<\frac{3C_{\sigma}\|f\|^2_{\cB}}{m}\right\}$, $E_2:=\left\{A_U<2\|f\|_{\cB}\right\}$. By Markov's inequality and taking $\epsilon\leq \|f\|_{\cB}^2/100$, we have
\begin{align*}
   \mathbb{P}(E_1)&=1-\mathbb{P}\left(\left\{L_U\ge\frac{3C_{\sigma}\|f\|^2_{\cB}}{m}\right\}\right)\ge 1-\frac{\mathbb{E}_U[L_U]}{3C_{\sigma}\|f\|^2_{\cB}/m}
   \ge 1-\frac{\|f\|^2_{\cB}+\epsilon}{3\|f\|^2_{\cB}}\ge \frac{199}{300},\\
   \mathbb{P}(E_2)&=1-\mathbb{P}\left(\left\{A_U\ge2\|f\|_{\cB}\right\}\right)\ge 1-\frac{\mathbb{E}_U[A_U]}{2\|f\|_{\cB}}
   \ge 1-\frac{\|f\|_{\cB}+\sqrt{\epsilon}}{2\|f\|_{\cB}}\ge \frac{9}{20}. 
\end{align*}
Therefore, the probability that two events happen together is
\begin{align*}
   \mathbb{P}(E_1\cap E_2)\ge \mathbb{P}(E_1)+\mathbb{P}(E_2)-1\ge \frac{199}{300}+\frac{9}{20}-1>\frac{1}{10}>0,
\end{align*}
which completes the proof by taking $\epsilon\to 0$.
\end{proof}

\subsection{Proof of Theorem \ref{thm: apriori-two-layer}}\label{sec: prf-apriori-two-layer}

Now we are ready to derive a priori estimates for the generalization error of two-layer neural networks with general activation functions. The proof is almost the same as~\cite{ma2019priori}, except for a different upper bound of the Rademacher complexity.

\subsubsection{A posteriori estimates}
According to Theorem $\ref{RadGenGap}$ in Appendix \ref{sec:BasicRad}, the Rademacher complexity can help bound the generalization gap. Combining with Theorem $\ref{Rad2-NN}$ we have the following a posteriori estimates.

\begin{theorem}\label{PosEst}
For any $\delta\in(0,1)$, with probability at least $1-\delta$ over the random training samples $\{\bx_i\}_{i=1}^n$, we have
\begin{align*}
   \left|R(\theta)-\hat{R}_n(\theta)\right|\le (\|\theta\|_{\tP}+1)\frac{2C'_\sigma\sqrt{2\ln{(2d+2)}}+1}{\sqrt{n}}+\sqrt{\frac{2\ln{(7/\delta)}}{n}},
\end{align*} 
where $C'_\sigma=4\gamma(\sigma)$.
\end{theorem}

\begin{proof}
Define $\mathcal{H}_Q=\left\{\ell(\bx,y;\theta):\|\theta\|_{\tP}\le Q\right\}$. For any $\{\bx_i\}_{i=1}^n$, let $\phi_i(y)=\frac{1}{2}\left|\mathcal{T}_{[0,1]}y-f^*(\bx_i)\right|^2$, $i=1,2,\cdots,n$, then it is easy to have
\begin{align*}
   |\phi_i(y)-\phi_i(y')|&=\frac{1}{2}\left|\mathcal{T}_{[0,1]}y-\mathcal{T}_{[0,1]}y'\right|\left|(\mathcal{T}_{[0,1]}y-f^*(\bx_i))+(\mathcal{T}_{[0,1]}y'-f^*(\bx_i))\right|\\
   &\le\frac{1}{2}\left|\mathcal{T}_{[0,1]}y-\mathcal{T}_{[0,1]}y'\right|\left(\left|\mathcal{T}_{[0,1]}y-f^*(\bx_i)\right|+\left|\mathcal{T}_{[0,1]}y'-f^*(\bx_i)\right|\right)\\
   &\le \left|\mathcal{T}_{[0,1]}y-\mathcal{T}_{[0,1]}y'\right|
   \le \left|y-y'\right|, 
\end{align*} 
i.e. $\phi_i(\cdot)$ is a 1-Lipschitz function. By the contraction lemma (Lemma $\ref{RadContraction}$), we have 
$
   \rad_n(\mathcal{H}_Q)\leq \rad_n(\mathcal{F}_Q).
$
By Theorem $\ref{RadGenGap}$, for any $\delta\in(0,1)$, with probability at least $1-\delta$ over the random training samples $\{\bx_i\}_{i=1}^n$,
\begin{align*}
   \sup_{\|\theta\|_{\tP}\le Q}\left|R(\theta)-\hat{R}_n(\theta)\right|
   \le 2\rad_n(\mathcal{H}_Q)+\sqrt{\frac{2\ln{(4/\delta)}}{n}}
   \le 2\rad_n(\mathcal{F}_Q)+\sqrt{\frac{2\ln{(4/\delta)}}{n}}.
\end{align*} 
According to Theorem $\ref{Rad2-NN}$,
\begin{align}\label{Rad2Res}
    \rad_n(\mathcal{F}_Q)\le C'_{\sigma}Q\sqrt{\frac{2\ln{(2d+2)}}{n}},
\end{align} 
where $C'_{\sigma}=4\gamma(\sigma)$. 
Therefore, with probability at least $1-\delta$,
\begin{align*}
   \sup_{\|\theta\|_{\tP}\le Q}\left|R(\theta)-\hat{R}_n(\theta)\right|
   \le 2C'_{\sigma}Q\sqrt{\frac{2\ln{(2d+2)}}{n}}+\sqrt{\frac{2\ln{(4/\delta)}}{n}}.
\end{align*} 

Now take $Q=1,2,\cdots$ and $\delta_Q=6\delta/(\pi Q)^2$, then with probability at least $1-\sum_{Q=1}^{\infty}\delta_Q=1-\delta$, the upper bound 
\begin{align*}
   \sup_{\|\theta\|_{\tP}\le Q}\left|R(\theta)-\hat{R}_n(\theta)\right|
   \le 2C'_{\sigma}Q\sqrt{\frac{2\ln{(2d+2)}}{n}}+\sqrt{\frac{2\ln{(4/\delta_Q)}}{n}}
\end{align*}
holds for all $Q\in\mathbb{N}^*$. Hence, for any given $\theta$, we can take $Q=\lceil\|\theta\|_{\tP}\rceil$, then $\|\theta\|_{\tP}\le Q<\|\theta\|_{\tP}+1$. Using the fact that $\sqrt{a+b}\le\sqrt{a}+\sqrt{b}$ for any $a,b\ge 0$ and $2\ln{t}<t$ for any $t\ge 1$, we have
\begin{align*}
   \left|R(\theta)-\hat{R}_n(\theta)\right|&\le\sup_{\|\theta\|_{\tP}\le Q}\left|R(\theta)-\hat{R}_n(\theta)\right|\\
   &\le 2C'_{\sigma}(\|\theta\|_{\tP}+1)\sqrt{\frac{2\ln{(2d+2)}}{n}}
   +\sqrt{\frac{2}{n}\ln{\left(\frac{2\pi^2}{3\delta}(\|\theta\|_{\tP}+1)^2\right)}}\\
   &\le 2C'_{\sigma}(\|\theta\|_{\tP}+1)\sqrt{\frac{2\ln{(2d+2)}}{n}}+\sqrt{\frac{2}{n}\ln{\left(\frac{2\pi^2}{3\delta}\right)}}+\sqrt{\frac{2\ln{\left(\|\theta\|_{\tP}+1\right)^2}}{n}}\\
   &\le 2C'_{\sigma}(\|\theta\|_{\tP}+1)\sqrt{\frac{2\ln{(2d+2)}}{n}}+\sqrt{\frac{2\ln{(7/\delta)}}{n}}+\frac{(\|\theta\|_{\tP}+1)}{\sqrt{n}}\\
   &=(\|\theta\|_{\tP}+1)\frac{2C'_{\sigma}\sqrt{2\ln{(2d+2)}}+1}{\sqrt{n}}+\sqrt{\frac{2\ln{(7/\delta)}}{n}},
\end{align*}
which completes the proof.
\end{proof}

\subsubsection{A priori estimates}
%\textbf{(2) A priori estimate}

\begin{proof}
Firstly, a direct estimate can be performed on $R(\hat{\theta}_n)$:
\begin{align*}
   R(\hat{\theta}_n)&=R(\tilde{\theta})+\left[R(\hat{\theta}_n)-R(\tilde{\theta})\right]\\
   &=R(\tilde{\theta})+\left[R(\hat{\theta}_n)-J(\hat{\theta}_n)\right]+\left[J(\hat{\theta}_n)-J(\tilde{\theta})\right]+\left[J(\tilde{\theta})-R(\tilde{\theta})\right]\\
   &\le R(\tilde{\theta})+\left[R(\hat{\theta}_n)-J(\hat{\theta}_n)\right]+\left[J(\tilde{\theta})-R(\tilde{\theta})\right],
\end{align*}
where the last inequality uses the fact that $J(\hat{\theta}_n)=\min_{\theta}J(\theta)$. Then
\begin{align*}
   R(\hat{\theta}_n)
   &\le R(\tilde{\theta})+\left[R(\hat{\theta}_n)-\hat{R}_n(\hat{\theta}_n)\right]+\left[\hat{R}_n(\hat{\theta}_n)-J(\hat{\theta}_n)\right]+\left[J(\tilde{\theta})-\hat{R}_n(\tilde{\theta})\right]+\left[\hat{R}_n(\tilde{\theta})-R(\tilde{\theta})\right]\\
   &\le R(\tilde{\theta})+\left|R(\hat{\theta}_n)-\hat{R}_n(\hat{\theta}_n)\right|-\lambda\|\hat{\theta}_n\|_{\tP}+\lambda\|\tilde{\theta}\|_{\tP}+\left|R(\tilde{\theta})-\hat{R}_n(\tilde{\theta})\right|.
\end{align*}
According to Theorem $\ref{PosEst}$, for any $\delta\in(0,1)$, with probability at least $1-\delta/2$ over the random training samples $\{\bx_i\}_{i=1}^n$,
\begin{align*}
   \left|R(\hat{\theta}_n)-\hat{R}_n(\hat{\theta}_n)\right|\le (\|\hat{\theta}_n\|_{\tP}+1)\frac{2C'_{\sigma}\sqrt{2\ln{(2d+2)}}+1}{\sqrt{n}}+\sqrt{\frac{2\ln{(14/\delta)}}{n}},
\end{align*} 
and 
\begin{align*}
   \left|R(\tilde{\theta})-\hat{R}_n(\tilde{\theta})\right|\le (\|\tilde{\theta}\|_{\tilde{P}}+1)\frac{2C'_{\sigma}\sqrt{2\ln{(2d+2)}}+1}{\sqrt{n}}+\sqrt{\frac{2\ln{(14/\delta)}}{n}}.
\end{align*} 
Thus, with probability at least $1-\delta$ over the random training samples $\{\bx_i\}_{i=1}^n$, the above two inequalities hold simultaneously. Therefore,
\begin{align}
   R(\hat{\theta}_n)&\le R(\tilde{\theta})+(\|\hat{\theta}_n\|_{\tP}+1)\lambda_n-\lambda\|\hat{\theta}_n\|_{\tP}+\lambda\|\tilde{\theta}\|_{\tP}+(\|\tilde{\theta}\|_{\tP}+1)\lambda_n+2\sqrt{\frac{2\ln{(14/\delta)}}{n}}\nonumber\\
   &=R(\tilde{\theta})+\lambda\|\tilde{\theta}\|_{\tP}+(\lambda_n-\lambda)\|\hat{\theta}_n\|_{\tP}+(\|\tilde{\theta}\|_{\tP}+2)\lambda_n+2\sqrt{\frac{2\ln{(14/\delta)}}{n}}\nonumber \\
   &\le R(\tilde{\theta})+\lambda\|\tilde{\theta}\|_{\tP}+(\|\tilde{\theta}\|_{\tP}+2)\lambda_n+2\sqrt{\frac{2\ln{(14/\delta)}}{n}}.\label{InterRhat}
\end{align}
Combining with Theorem $\ref{2-NNApp}$ yields $(\ref{2-NNPriorEst})$, which completes the proof.
\end{proof}

\section{Proofs for Section~\ref{SecResRad}}

\subsection{Proofs for Theorem \ref{RadResNet}}\label{sec: pfSecResRad}

\subsubsection{Extension of norms to hidden neurons} \label{sec:ExtHidNeu}

To use the method of induction to prove Theorem \ref{RadResNet}, firstly we need to extend the definition of modified weighted path norm~\eqref{PathRes} (equivalently (\ref{PathResr})-(\ref{PathResM})) to the hidden neurons in the residual network.

\begin{definition}\label{eqn: def-norm-resnetHid}
Given a residual network defined by $(\ref{ResNet})$, let
\begin{align*}
\bg_l=\sigma(W_l\bh_{l-1}),\ l=1,2,\cdots,L,
\end{align*}
and $g_l^i$ be the $i$-th element of $\bg_l$. Define its norm to be 
\begin{align}\label{PathResHid}
\|g_l^i\|_{\tP}=c\sum_{k=0}^{l-1}\||W_l^{i,:}|(I+c|U_{l-1}||W_{l-1}|)\cdots(I+c|U_{k+1}||W_{k+1}|)|U_k|\|_1,
\end{align}
with the convention that $A_{l-1}A_{l-2}\cdots A_{k+1}=I$ when $k\ge l-1$. Here $U_0:=V$, $W_l^{i,:}$ is the $i$-th row of $W_l$, and $c:=c_{\sigma}>0$ is an absolute constant only related to the activation function.
\end{definition}

Obviously this definition is a natural extension of the norm~\eqref{PathRes} to hidden neurons, and consistent with \eqref{PathRes}. Next, we give the following recursive form of the norm~\eqref{PathResHid}, which will be used in the proof of Theorem~\ref{RadResNet}.

\begin{theorem}\label{thm: recursive_path_resHid}
The norm~\eqref{PathResHid} can be written as a modification of the weighted path norm for hidden neurons $\|g_l^i\|_{\cP}$ defined in~\cite{ma2019priori}:
\begin{align}
\|g_l^i\|_{\tP}&=c\||W_l^{i,:}|(I+c|U_{l-1}||W_{l-1}|)\cdots(I+c|U_1||W_1|)|V|\|_1\label{WtPthNrm}\\
&\quad + c\sum_{k=1}^{l-1}\||W_l^{i,:}|(I+c|U_{l-1}||W_{l-1}|)\cdots(I+c|U_{k+1}||W_{k+1}|)|U_k|\|_1\label{Modli}\\
&:=\|g_l^i\|_{\cP}+ M_{l,i},\label{Decompsition}
\end{align}
where the modification term $M_{l,i}$ satisfies
\begin{align}\label{PathResMHidi}
M_{l,i}=c|W_{l}^{i,:}|\sum_{k=1}^{l-1}|U_k|(\bM_k+\bm{1}_m),\quad l=2,3,\cdots,L, \ i=1,2,\cdots,m,
\end{align}
with $\bM_1=\bm{0}_m$. Here $\bM_k:=(M_{k,1},M_{k,2},\cdots,M_{k,m})^T$ for $k=1,2,\cdots,l$.
\end{theorem}

\begin{proof}
The fact that $\bM_1=\bm{0}_m$ is obvious. We need to prove the equivalence of $(\ref{Modli})$ and $(\ref{PathResMHidi})$, i.e. 
\begin{align}
\bM_{l+1}&=c|W_{l+1}|\sum_{k=1}^{l}|U_k|(\bM_k+\bm{1}_m),\quad\bM_1=\bm{0}_m,\label{PathResMHid}\\
\Leftrightarrow\bM_{l+1}&=c|W_{l+1}|\sum_{k=1}^{l}(I+c|U_{l}||W_{l}|)(I+c|U_{l-1}||W_{l-1}|)\cdots(I+c|U_{k+1}||W_{k+1}|)|U_k|\bm{1}_m\label{Modl}
\end{align}
for $l=1,2,\cdots,L-1$. Denote 
\begin{align*}
Z_{l+1}=\sum_{k=1}^{l}(I+c|U_l||W_l|)(I+c|U_{l-1}||W_{l-1}|)\cdots(I+c|U_{k+1}||W_{k+1}|) |U_k|,
\end{align*}
and $Z_1=0$. Then we have $Z_2=|U_1|$ by convention, and
\begin{align*}
Z_{l+1}&=(I+c|U_l||W_l|)\sum_{k=1}^{l-1}(I+c|U_{l-1}||W_{l-1}|)\cdots(I+c|U_{k+1}||W_{k+1}|) |U_k|+|U_l|\\
&=(I+c|U_l||W_l|)Z_l+|U_l|,
\end{align*}
or
\begin{align}\label{ResZ}
Z_{l+1}-Z_l=|U_l|(c|W_l|Z_l+I).
\end{align}

\textbf{(i)}~(\ref{PathResMHid})$\Rightarrow$(\ref{Modl}). The aim is to prove 
\begin{align}\label{ModlZ}
\bM_{l+1}=c|W_{l+1}|Z_{l+1}\bm{1}_m,\quad l=1,2,\cdots,L-1, 
\end{align}
This is done by induction. 

For $l=1$, since $\bM_1=\bm{0}_m$, (\ref{PathResMHid}) gives $\bM_2=c|W_2||U_1|\bm{1}_m=c|W_2|Z_2\bm{1}_m$. Assume that $\bM_k=c|W_k|Z_k\bm{1}_m$ for $k=1,2,\cdots,l$, applying $(\ref{ResZ})$ to have
\begin{align*}
\bM_{l+1}&=c|W_{l+1}|\sum_{k=1}^l|U_k|(\bM_k+\bm{1}_m)\\
&=c|W_{l+1}|\sum_{k=1}^l|U_k|(c|W_k|Z_k+I)\bm{1}_m\\
&=c|W_{l+1}|\sum_{k=1}^l(Z_{k+1}-Z_k)\bm{1}_m\\
&=c|W_{l+1}|Z_{l+1}\bm{1}_m.
\end{align*}

\textbf{(ii)}~(\ref{Modl})$\Rightarrow$(\ref{PathResMHid}). Assume that $\bM_{k}=c|W_{k}|Z_{k}\bm{1}_m$ for $k=2,3,\cdots,L$.\footnote{This assumption is also consistent with $k=1$ by the convention that $Z_1=0$.} Again, by $(\ref{ResZ})$, for any $l=1,2,\cdots,L-1$, we have
\begin{align*}
&c|W_{l+1}|\sum_{k=1}^l|U_k|(\bM_k+\bm{1}_m)\\
=~&c|W_{l+1}|\sum_{k=1}^l|U_k|(c|W_k|Z_k+I)\bm{1}_m\\
=~&c|W_{l+1}|\sum_{k=1}^l(Z_{k+1}-Z_k)\bm{1}_m\\
=~&c|W_{l+1}|Z_{l+1}\bm{1}_m=\bM_{l+1}.
\end{align*}

Combining \textbf{(i)} and \textbf{(ii)} finishes the proof. 
\end{proof}

\begin{remark}\label{rmk:prfmodsame}
   The $\bM_{l}$ defined in Theorem \ref{thm: recursive_path_resHid} is exactly the desired $\bM_{l}$ in Lemma \ref{thm: recursive_path_res}. As a result, the proof of Lemma \ref{thm: recursive_path_res} is totally the same as above. That is to say, the recursive definition (\ref{PathResr})-(\ref{PathResM}) is equivalent to the closed form~\eqref{PathRes}.
\end{remark}

\begin{remark}
   Theorem \ref{thm: recursive_path_resHid} and Lemma \ref{thm: recursive_path_res} show that $\bM_{l}$ and $r$ can be viewed as the modification of the weighted path norm (see \cite{ma2019priori}) at the $l$-th layer and the output, respectively. In addition, this modification captures the effect of all bias, and can be obtained recursively.
\end{remark}

\begin{remark}
Since the modification is caused by the approximation for activation functions (utilizing two-layer ReLU networks), it does not occur until the first activation. Therefore, the modified terms~(\ref{PathResr}) and (\ref{Modli}) do not contain $V$ and $W_1$, since they are both before the first activation.
\end{remark}

The following lemmas establish the relationship between $\|\theta\|_{\cP}$ and $\|g_l^i\|_{\cP}$. It is in fact similar to the ReLU case handled in~\cite{ma2019priori}, only with a change of weight factor from $3$ to any constant $c>0$.

\begin{lemma}\label{GPathRel}
For the weighted path norm $\|\theta\|_{\cP}$ and $\|g_l^i\|_{\cP}$, we have 
\begin{align}
\|\theta\|_{\cP}=\sum_{l=1}^L\sum_{j=1}^m\left(|\bm{\alpha}|^T|U_l^{:,j}|\right)\|g_l^j\|_{\cP}+\||\bm{\alpha}|^T|V|\|_1,\label{PathNormRel1}
\end{align}
and
\begin{align}
\|g_l^i\|_{\cP}=c\sum_{k=1}^{l-1}\sum_{j=1}^m\left(|W_l^{i,:}||U_k^{:,j}|\right)\|g_k^j\|_{\cP}+c\||W_l^{i,:}||V|\|_1,\label{PathNormRel2}
\end{align}
where $U_l^{:,j}$ is the $j$-th column of $U_l$.
\end{lemma}

\begin{proof}
Denote $Z_l=(I+c|U_{l}||W_{l}|)(I+c|U_{l-1}||W_{l-1}|)\cdots(I+c|U_1||W_1|)$, $l=1,2,\cdots,L$, and $Z_0=I$. Then we have
\begin{align*}
Z_l&=(I+c|U_{l}||W_{l}|)Z_{l-1}=c|U_{l}||W_{l}|Z_{l-1}+Z_{l-1}\\
&=c|U_{l}||W_{l}|Z_{l-1}+c|U_{l-1}||W_{l-1}|Z_{l-2}+Z_{l-2}\\
&=\cdots=c\sum_{i=1}^l|U_i||W_i|Z_{i-1}+I.
%\\&=c\sum_{i=1}^l|U_i||W_i|\prod_{j=1}^{i-1}(I+c|U_{i-j}||W_{i-j}|)+I.
\end{align*}
%Rewrite the definition of $\|\theta\|_{P}$ and $\|g_l^i\|_P$ to be
%\begin{align*}
%\|\theta\|_{P}&=\||u|^T(I+c|U_L||W_L|)\cdots(I+c|U_1||W_1|)|V|\|_1=\||u|^TZ_L|V|\|_1,\\
%\|g_l^i\|_P&=c\||W_l^{i,:}|(I+c|U_{l-1}||W_{l-1}|)\cdots(I+c|U_1||W_1|)|V|\|_1=c\||W_l^{i,:}|Z_{l-1}|V|\|_1,
%\end{align*}
Therefore
\begin{align*}
\|\theta\|_{\cP}&=\||\bm{\alpha}|^TZ_L|V|\|_1=\left\||\bm{\alpha}|^T\left(c\sum_{l=1}^L|U_l||W_l|Z_{l-1}+I\right)|V|\right\|_1\\
&=\sum_{l=1}^Lc\left\||\bm{\alpha}|^T|U_l||W_l|Z_{l-1}|V|\right\|_1+\left\||\bm{\alpha}|^T|V|\right\|_1\\
&=\sum_{l=1}^Lc\left\|\sum_{j=1}^m\left(|\bm{\alpha}|^T|U_l^{:,j}|\right)|W_l^{j,:}|Z_{l-1}|V|\right\|_1+\left\||\bm{\alpha}|^T|V|\right\|_1\\
&=\sum_{l=1}^L\sum_{j=1}^m\left(|\bm{\alpha}|^T|U_l^{:,j}|\right)\cdot c\left\||W_l^{j,:}|Z_{l-1}|V|\right\|_1+\left\||\bm{\alpha}|^T|V|\right\|_1\\
&=\sum_{l=1}^L\sum_{j=1}^m\left(|\bm{\alpha}|^T|U_l^{:,j}|\right)\cdot \|g_l^j\|_{\cP}+\left\||\bm{\alpha}|^T|V|\right\|_1,
\end{align*}
which gives $(\ref{PathNormRel1})$. The proof of $(\ref{PathNormRel2})$ is similar.
\end{proof}

The following two lemmas will be repeatedly utilized in the Rademacher calculus of residual networks (the proof of Theorem \ref{RadResNet}).

\begin{lemma}\label{RelG1}
Let $G_l^Q=\{g_l^i:\|g_l^i\|_{\tP}\le Q\}$. Then we have $G_{l'}^Q\subset G_l^Q$ for $l'\le l$.
\end{lemma}

\begin{proof}
For any $g_{l'}^i\in G_{l'}^Q$, let $V,W_1,U_1,\cdots,W_{l'-1},U_{l'-1},W_{l'}^{i,:}$ be the parameters of $g_{l'}^i$. For any $l\ge l'$, consider $g_l^j$ generated by parameters $V,W_1,U_1,\cdots,W_{l'-1},U_{l'-1},W_{l'},U_{l'},\cdots,W_{l-1},U_{l-1},W_{l}^{j,:}$ with $W_k=U_k=0$ for $k=l',l'+1,\cdots,l-1$ and $W_{l}^{j,:}=W_{l'}^{i,:}$. Then it is easy to verify $g_l^j=g_{l'}^i$ and $\|g_{l'}^i\|_{\cP}=\|g_l^j\|_{\cP}$. By (\ref{PathResMHidi}), notice that
\begin{align*}
M_{l,j}=c|W_{l}^{j,:}|\sum_{k=1}^{l-1}|U_k|(\bM_k+\bm{1}_m)=c|W_{l'}^{i,:}|\sum_{k=1}^{l'-1}|U_k|(\bM_k+\bm{1}_m)=M_{l',i},
\end{align*}
we have $\|g_{l}^j\|_{\tP}=\|g_{l'}^i\|_{\tP}\le Q$. That is to say, $g_{l'}^i=g_{l}^j\in G_{l}^Q$.
\end{proof}

For convenience, we also write $\bg_l'=W_l\bh_{l-1}$ for $l=1,2,\cdots,L$, i.e. $\bg_l=\sigma(\bg_l')$, and $(g_l^{i})'$ be the $i$-th element of $\bg_l'$.

\begin{lemma}\label{RelG2Sca}
Let $(G_l^{Q})'=\{(g_l^{i})':\|(g_l^{i})'\|_{\tP}\le Q\}$, then $(G_l^{q})'\subset (G_l^{Q})'$ and $(G_l^{q})'=\frac{q}{Q} (G_l^{Q})'$.
\end{lemma}

\begin{proof}
Obviously we have $(G_l^q)'\subset (G_l^Q)'$ for any $q\le Q$. Notice that both the two parts of $\|(g_l^i)'\|_{\tP}$ ($\|(g_l^i)'\|_{\cP}$ and $M_{l,i}$) have an output parameter $W_{l}^{i,:}$, meaning that the scaling process can be done on the row vector $W_{l}^{i,:}$, we can easily have $(G_l^{q})'=\frac{q}{Q} (G_l^{Q})'$. 

In fact, for any $(g_{l}^i)'\in (G_{l}^Q)'$, define $(\tilde{g}_l^i)'$ by replacing the output parameter $W_{l}^{i,:}$ by $\frac{q}{Q}W_{l}^{i,:}$, then we have $(\tilde{g}_l^i)'=\frac{q}{Q}(g_{l}^i)'$, and by (\ref{Decompsition}),  $$\|(\tilde{g}_l^i)'\|_{\tP}=\|(\tilde{g}_l^i)'\|_{\cP}+\tilde{M}_{l,i}=\frac{q}{Q}\|(g_l^i)'\|_{\cP}+\frac{q}{Q}M_{l,i}=\frac{q}{Q}\|(g_{l}^i)'\|_{\tP}\le q,$$
hence $(\tilde{g}_{l}^i)'\in (G_{l}^q)'$. Therefore we have $\frac{q}{Q}(G_{l}^Q)'\subset (G_{l}^q)'$. Similarly we can obtain $\frac{Q}{q}(G_{l}^q)'\subset (G_{l}^Q)'$. As a result, $(G_l^{q})'=\frac{q}{Q} (G_l^{Q})'$. 
\end{proof}
\begin{comment}
The proof is the same as ``$\|g_l^i\|_{P}$'' case, only to notice that: 

(i) for any $l'>l$, if we take the parameters $V,W_1,U_1,\cdots,W_{l-1},U_{l-1},W_l,U_l,\cdots,W_{l'-1},U_{l'-1},W_{l'}^{i,:}$ to be $W_i=U_i=0$ for $i=l,l+1,\cdots,l'-1$ and $W_{l'}^{i,:}=W_{l}^{i,:}$, we then have  
\begin{align*}
M_{l',i}=c|W_{l'}^{i,:}|\sum_{k=1}^{l'-1}|U_k|(M_k+\bm{1}_m)=|W_{l}^{i,:}|\sum_{k=1}^{l-1}|U_k|(cM_k+\bm{1}_m)=M_{l,i};
\end{align*}

(ii) Both the two parts of $\|g_l^i\|_{\tilde{P}}$ ($\|g_l^i\|_{P}$ and $M_{l,i}$) have an output parameter $W_{l}^{i,:}$, meaning that the scaling process can be done on the row vector $W_{l}^{i,:}$.
\end{comment}

The following simple lemma shows the Lipschitz continuity of the activation function $\sigma(\cdot)$.
\begin{lemma}\label{ActLip}
Assume that the activation function $\sigma(\cdot)$ satisfies the conditions in Theorem $\ref{1DApp}$. Then $\sigma(\cdot)$ is a Lipschitz continuous function on $\bR^m$ with the Lipschitz constant $L_{\sigma}$ satisfying $L_{\sigma}\le \gamma(\sigma)+\min\{|\sigma'(+\infty)|, |\sigma'(-\infty)|\}$.
\end{lemma}

\begin{proof}
Lemma $\ref{Asymptote}$ shows that $\sigma'(+\infty)$ and $\sigma'(-\infty)$ exist, therefore, by Corollary $\ref{cly:2}$, we have 
\begin{align*}
|\sigma'(x)-\sigma'(+\infty)| &
= \left|\int_{x}^{+\infty}\sigma''(t)dt\right|
\le \int_{\bR} |{\sigma}''(t)|dt\le \gamma_0(\sigma) \le \gamma(\sigma),\\
|\sigma'(x)-\sigma'(-\infty)| &
= \left|\int^{x}_{-\infty}\sigma''(t)dt\right|
\le \int_{\bR} |{\sigma}''(t)|dt\le \gamma_0(\sigma) \le \gamma(\sigma).
\end{align*}
Hence $|\sigma'(x)|\leq \gamma(\sigma) + \min\{|\sigma'(+\infty)|,|\sigma'(-\infty)|\}$
\end{proof}

\subsubsection{Proof of Theorem~\ref{RadResNet}}\label{sec:ProofRad}
%\noindent\textbf{3. Proof of Theorem~\ref{RadResNet}}

Based on Lemma \ref{RadLin} and Lemma \ref{RadContraction} in Appendix \ref{sec:BasicRad}, and Appendix \ref{sec:ExtHidNeu}, we can now get down to prove Theorem~\ref{RadResNet}.
\begin{proof}
The crucial step is to inductively estimate the Rademacher complexity of $G_l^Q$. We will show that
\begin{equation}\label{RadG}
\rad_n(G_l^Q)\le Q\sqrt{\frac{2\ln{(2d+2)}}{n}},\ l=1,2,\cdots,L.
\end{equation}
\\ \hspace*{\fill} \\
\noindent\textbf{(1)~The first layer: a standard analysis}

For $l=1$, according to the contraction lemma (Lemma $\ref{RadContraction}$) and the Rademacher complexity of linear functions (Lemma $\ref{RadLin}$), we have
\begin{align*}
n\rad_n(G_1^Q)&=\mathbb{E}_{\bxi}\sup_{g_1\in G_1^Q}\sum_{i=1}^n \xi_ig_1(\tilde{\bx}_i)\\
&=\mathbb{E}_{\bxi}\sup_{g_1\in G_1^Q}\sum_{i=1}^n \xi_i\sigma(\bw_1^T V\tilde{\bx}_i)\\
&\le L_{\sigma}\mathbb{E}_{\bxi}\sup_{g_1\in G_1^Q}\sum_{i=1}^n \xi_i \bw_1^T V\tilde{\bx}_i\\
&= L_{\sigma}\mathbb{E}_{\bxi}\sup_{c\||\bw_1|^T|V|\|_1\le Q}\sum_{i=1}^n \xi_i \||\bw_1|^T |V|\|_1 \hat{\bw}^T\tilde{\bx}_i \qquad (\|\hat{\bw}\|_1=1)\\
& \le L_{\sigma}\mathbb{E}_{\bxi}\sup_{c\||\bw_1|^T|V|\|_1\le Q}\||\bw_1|^T |V|\|_1\sup_{\|\hat{\bw}\|_1\le 1}\sum_{i=1}^n \xi_i\hat{\bw}^T\tilde{\bx}_i\\
& \le \frac{L_{\sigma}}{c}Q\mathbb{E}_{\bxi}\sup_{\|\hat{\bw}\|_1\le 1}\sum_{i=1}^n \xi_i\hat{\bw}^T\tilde{\bx}_i\\
&\le n Q\sqrt{\frac{2\ln{(2d+2)}}{n}},
\end{align*}
as long as $c\ge L_{\sigma}$.
\\ \hspace*{\fill} \\
\noindent\textbf{(2)~The general layer: approximation and decomposition}

Assume the result $(\ref{RadG})$ holds for $1,2,\cdots,l$, the aim is to prove that $(\ref{RadG})$ holds for $l+1$. Consider the following dynamics
\begin{align}\label{ResNetApp}
\begin{split}
\tilde{\bh}_0&=V\tilde{\bx},\\
\tilde{\bg}_i&=\tilde{\bm{\sigma}}_i(W_i\tilde{\bh}_{i-1}),\\
\tilde{\bh}_i&=\tilde{\bh}_{i-1}+U_i\tilde{\bg}_i,\ i=1,2,\cdots,l,\\
\tilde{g}_{l+1}^{s}&=\sigma(W_{l+1}^{s,:}\tilde{\bh}_l),
\end{split}
\end{align}
where $\tilde{\bm{\sigma}}_i:=(\tilde{\sigma}_{i,k})$, $i=1,2,\cdots,l$, $k=1,2,\cdots,m$, are some two-layer ReLU networks which can vary from different neurons at different layers, and the activation is operated in the same indices. According to Theorem $\ref{1DApp}$, the dynamics $(\ref{ResNetApp})$ can be seen as an approximation of ResNet $(\ref{ResNet})$ up to the $(l+1)$-th layer by selecting appropriate $\tilde{\sigma}(\cdot)$. Then consider the decomposition 
\begin{align*}
n\rad_n(G_{l+1}^Q)&=
\mathbb{E}_{\bxi}\sup_{g_{l+1}\in G_{l+1}^Q}
\sum_{i=1}^n \xi_ig_{l+1}(\tilde{\bx}_i)\\
&=\mathbb{E}_{\bxi}\sup_{(1)}
\left\{\sum_{i=1}^n \xi_i[g_{l+1}(\tilde{\bx}_i)-\tilde{g}_{l+1}(\tilde{\bx}_i)]+\sum_{i=1}^n\xi_i\tilde{g}_{l+1}(\tilde{\bx}_i)\right\}\\
&\le\mathbb{E}_{\xi}\sup_{(1)}
\sum_{i=1}^n \xi_i[g_{l+1}(\tilde{\bx}_i)-\tilde{g}_{l+1}(\tilde{\bx}_i)]+\mathbb{E}_{\xi}\sup_{(1)} \sum_{i=1}^n\xi_i\tilde{g}_{l+1}(\tilde{\bx}_i)\\
&=\mathbb{E}_{\xi}\sup_{(1)}
\sum_{i=1}^n \xi_i[g_{l+1}(\tilde{\bx}_i)-\tilde{g}_{l+1}(\tilde{\bx}_i)]+\mathbb{E}_{\xi}\sup_{\tilde{g}_{l+1}\in G_{l+1}^Q} \sum_{i=1}^n\xi_i\tilde{g}_{l+1}(\tilde{\bx}_i):=I_1+I_2,
\end{align*}
where condition $(1)$ is 
$$\|g_{l+1}^i\|_{\cP}+M_{l+1,i}=c\||W_{l+1}^{i,:}|(I+c|U_{l}||W_{l}|)\cdots(I+c|U_1||W_1|)|V|\|_1+M_{l+1,i}\le Q$$ 
according to (\ref{Decompsition}). Since $M_{l+1,i}$ is a function of parameters $U_1,W_2,U_2,\cdots,W_{l},U_{l},W_{l+1}^{i,:}$, just like $\|g_{l+1}^i\|_{\cP}$, the supreme is in fact taken over parameters up to the $(l+1)$-th layer which satisfy the (modified weighted path) norm control. Notice that ${g}_{l+1}$ and $\tilde{g}_{l+1}$ enjoy the same parameters, the last equality holds. 

Next we need to bound $I_1$ and $I_2$ respectively. The basic framework is as follows:
\begin{itemize}
    \item control the ``approximation error'' $I_1$ to be arbitrarily small (i.e. smaller than a pre-selected tolerance error $\epsilon$), by selecting appropriate $\tilde{\sigma}(\cdot)$;
    \item derive a ``uniform'' upper bound for $I_2$ like $(\ref{RadG})$ for any $\epsilon$-approximation dynamics $(\ref{ResNetApp})$, by taking appropriate weight $c_{\sigma}$ (according to different $\epsilon$);
    \item combining the above and letting $\epsilon\rightarrow 0$ yield the desired conclusion.
\end{itemize}

\textbf{(2a) The error term $I_1$}

Fix any $g_{l+1}\in G_{l+1}^Q$, or parameters $\{V,W_1,U_1,\cdots,W_l,U_l,W_{l+1}^{s,:}\}$. Let the dynamics $(\ref{ResNetApp})$ and $(\ref{ResNet})$ evolve from the same data $\bx\in X$, then $\bh_0=\tilde{\bh}_0$. Define $\be_i=\bh_i-\tilde{\bh}_i$, $i=0,1,\cdots,l$. Notice that 
\begin{align}
\bh_i&=\bh_{i-1}+U_i{\sigma}(W_i{\bh}_{i-1}),\label{h}\\
\tilde{\bh}_i&=\tilde{\bh}_{i-1}+U_i\tilde{\bm{\sigma}}_i(W_i\tilde{\bh}_{i-1}),\label{h'}
\end{align}
$(\ref{h})-(\ref{h'})$ gives
\begin{align*}
\be_i&=\be_{i-1}+U_i[{\sigma}(W_i{\bh}_{i-1})-\tilde{\bm{\sigma}}_i(W_i\tilde{\bh}_{i-1})]\\
&=\be_{i-1}+U_i\left[({\sigma}(W_i{\bh}_{i-1})-{\sigma}(W_i\tilde{\bh}_{i-1}))+({\sigma}(W_i\tilde{\bh}_{i-1})-\tilde{\bm{\sigma}}_i(W_i\tilde{\bh}_{i-1}))\right].
\end{align*}
Then we have 
\begin{align}\label{ErrM1}
\|\be_i\|_2&\le\|\be_{i-1}\|_2+\|U_i\|_2\left(\|{\sigma}(W_i{\bh}_{i-1})-{\sigma}(W_i\tilde{\bh}_{i-1})\|_2+\|{\sigma}(W_i\tilde{\bh}_{i-1})-\tilde{\bm{\sigma}}_i(W_i\tilde{\bh}_{i-1})\|_2\right).
\end{align}

Fix $\epsilon\in(0,1)$ arbitrarily. Set 
\begin{align}\label{tol}
\epsilon_{i,k}=\epsilon\left(2^i\sqrt{m}L_{\sigma}\max\{\|W_{l+1}^{s,:}\|_2,1\}\max\{\|U_i\|_2,1\}\prod_{j=i+1}^l(1+L_{\sigma}\|U_j\|_2\|W_j\|_2)\right)^{-1} 
\end{align}
for $i=1,2,\cdots,l$ and $k=1,2,\cdots,m$. According to Theorem $\ref{1DApp}$, we can build a sequence of two-layer ReLU networks $\{\tilde{\sigma}_{i,k}(\cdot)\}$ that satisfies
\begin{align}
&\|\sigma(y)-\tilde{\sigma}_{i,k}(y)\|_\infty\le\epsilon_{i,k},\\
&\|\theta(\tilde{\sigma}_{i,k})\|_{\cP}\le\gamma(\sigma)+\epsilon_{i,k}\le \gamma(\sigma)+\epsilon.\label{2NNPthNrm}
\end{align}
Let $\bm{\epsilon}_i=(\epsilon_{i,1},\epsilon_{i,2},\cdots,\epsilon_{i,m})^T$, recall $\tilde{\bm{\sigma}}_i:=(\tilde{\sigma}_{i,1},\tilde{\sigma}_{i,2},\cdots,\tilde{\sigma}_{i,m})^T$, we have
\begin{align}\label{tolv}
\|\|\sigma\bm{1}_m-\tilde{\bm{\sigma}}_{i}\|_\infty\|_2\le\|\bm{\epsilon}_{i}\|_2=\sqrt{m}\epsilon_{i,k},
\end{align}
where $\|\|\bm{f}(\bx)\|_{\infty}\|_2:=\|(\|(f_1(x_1)\|_{\infty},\|(f_2(x_2)\|_{\infty},\cdots,\|(f_m(x_m)\|_{\infty})\|_2$ for any $\bm{f}: \bR^m\mapsto\bR^m$.

Now we can continue to estimate $(\ref{ErrM1})$:
\begin{align*}
\|\be_l\|_2&\le\|\be_{l-1}\|_2+\|U_l\|_2\left(L_{\sigma}\|W_l{\bh}_{l-1}-W_l\tilde{\bh}_{l-1}\|_2+\|\|{\sigma}\bm{1}_m-\tilde{\bm{\sigma}}_{l}\|_{\infty}\|_2\right)\\
&\le \|\be_{l-1}\|_2+\|U_l\|_2\left(L_{\sigma}\|W_l\|_2\|{\be}_{l-1}\|_2+\|\bm{\epsilon}_{l}\|_2\right)\\
&=\|\be_{l-1}\|_2(1+L_{\sigma}\|U_l\|_2\|W_l\|_2)+\|U_{l}\|_2\|\bm{\epsilon}_{l}\|_2.
\end{align*}
Using this inequality recursively, we finally obtain that
\begin{align*}
\|\be_l\|_2&\le\prod_{i=1}^l(1+L_{\sigma}\|U_i\|_2\|W_i\|_2)\|\be_0\|_2+
\sum_{i=1}^l \|U_{i}\|_2\|\bm{\epsilon}_{i}\|_2\prod_{j=i+1}^l(1+L_{\sigma}\|U_j\|_2\|W_j\|_2),
\end{align*}
with the convention that $\prod_{k=i}^j c_k=1$ when $i>j$. Notice that $\be_0=\bh_0-\tilde{\bh}_0=\bm{0}_D$, and the selection of $\bm{\epsilon}_i$ $(\ref{tol})$ and $(\ref{tolv})$, we have
\begin{align*}
\|\be_l\|_2&\le\sum_{i=1}^l\frac{\epsilon}{2^iL_{\sigma}\max\{\|W_{l+1}^{s,:}\|_2,1\}}\le\frac{\epsilon}{L_{\sigma}\max\{\|W_{l+1}^{s,:}\|_2,1\}}.
\end{align*}
Then applying the Cauchy-Schwarz inequality to have
\begin{align*}
|\tilde{g}_{l+1}^{s}-{g}_{l+1}^{s}|=|\sigma(W_{l+1}^{s,:}\tilde{\bh}_l)-\sigma(W_{l+1}^{s,:}{\bh}_l)|\le L_{\sigma}|W_{l+1}^{s,:}\be_l|\le L_{\sigma} \|W_{l+1}^{s,:}\|_2\|\be_l\|_2\le\epsilon,
\end{align*}
therefore  
\begin{align*}
|I_1|\le\mathbb{E}_{\bxi}\sup_{g_{l+1}\in G_{l+1}^Q}
\sum_{i=1}^n |\xi_i||g_{l+1}(\tilde{\bx}_i)-\tilde{g}_{l+1}(\tilde{\bx}_i)|\le n\epsilon.
\end{align*}

\textbf{(2b) The main term $I_2$}

By the contraction lemma (Lemma $\ref{RadContraction}$) and approximation dynamics $(\ref{ResNetApp})$, we have that 
\begin{align}
I_2&=\mathbb{E}_{\bxi}\sup_{\tilde{g}_{l+1}\in G_{l+1}^Q} \sum_{t=1}^n\xi_t\tilde{g}_{l+1}(\tilde{\bx}_t) \nonumber\\
&= \mathbb{E}_{\bxi}\sup_{\tilde{g}_{l+1}\in G_{l+1}^Q} \sum_{t=1}^n\xi_t\sigma(\bw_{l+1}^T\tilde{\bh}_l) \nonumber\\
&\le L_{\sigma}\mathbb{E}_{\bxi}\sup_{\tilde{g}_{l+1}\in G_{l+1}^Q} \sum_{t=1}^n\xi_t\left[\bw_{l+1}^T\left(\sum_{i=1}^l(\tilde{\bh}_i-\tilde{\bh}_{i-1})+\tilde{\bh}_0\right)\right] \label{Activ} \\
&=L_{\sigma}\mathbb{E}_{\bxi}\sup_{\tilde{g}_{l+1}\in G_{l+1}^Q}\left\{ \sum_{t=1}^n\xi_t \bw_{l+1}^T\sum_{i=1}^lU_i\tilde{\bg}_i+\sum_{t=1}^n\xi_t \bw_{l+1}^TV\tilde{\bx}_t\right\} \nonumber\\
&\le L_{\sigma}\mathbb{E}_{\bxi}\sup_{\tilde{g}_{l+1}\in G_{l+1}^Q}\sum_{t=1}^n\xi_t \bw_{l+1}^T\sum_{i=1}^lU_i\tilde{\bg}_i+L_{\sigma}\mathbb{E}_{\bxi}\sup_{\tilde{g}_{l+1}\in G_{l+1}^Q}\sum_{t=1}^n\xi_t \bw_{l+1}^TV\tilde{\bx}_t:=L_{\sigma}I_{2,1}+L_{\sigma}I_{2,2}. \nonumber
\end{align}

\textbf{(2b-i)~Bound for $I_{2,2}$}

The analysis is the same as the case $l=1$ (see Appendix \ref{sec:ProofRad} \textbf{(1)}), where the Rademacher complexity of linear functions (Lemma $\ref{RadLin}$) is used, since $\tilde{g}_{l+1}\in G_{l+1}^Q$ implies that $Q\ge \|\tilde{g}_{l+1}\|_{\tP}\ge \|\tilde{g}_{l+1}\|_{\cP}\ge c\||\bw_{l+1}|^T|V|\|_1$:
\begin{align*}
I_{2,2}&=\mathbb{E}_{\bxi}\sup_{\tilde{g}_{l+1}\in G_{l+1}^Q}\||\bw_{l+1}|^T|V|\|_1\sum_{t=1}^n\xi_t \hat{\bw}_{l+1}^T\tilde{\bx}_t \qquad (\|\hat{\bw}_{l+1}\|_1=1)\\
&\le\mathbb{E}_{\bxi}\sup_{\tilde{g}_{l+1}\in G_{l+1}^Q}\||\bw_{l+1}|^T|V|\|_1\sup_{\|\hat{\bw}\|_1\le 1}\sum_{t=1}^n\xi_t \hat{\bw}^T\tilde{\bx}_t \\
&\le\frac{Q}{c}\mathbb{E}_{\bxi}\sup_{\|\hat{\bw}\|_1\le 1}\sum_{t=1}^n\xi_t \hat{\bw}^T\tilde{\bx}_t\\
&\le n \frac{Q}{c}\sqrt{\frac{2\ln{(2d+2)}}{n}}.
\end{align*}

\textbf{(2b-ii)~Bound for $I_{2,1}$ }

We firstly have
\begin{align*}
I_{2,1}&=\mathbb{E}_{\bxi}\sup_{\tilde{g}_{l+1}\in G_{l+1}^Q}\sum_{t=1}^n\xi_t \bw_{l+1}^T\sum_{i=1}^l\sum_{k=1}^mU_i^{:,k}\tilde{g}_i^k(\tilde{\bx}_t)\\
&=\mathbb{E}_{\bxi}\sup_{\tilde{g}_{l+1}\in G_{l+1}^Q}\sum_{i=1}^l\sum_{k=1}^m \bw_{l+1}^TU_i^{:,k}\sum_{t=1}^n\xi_t\tilde{g}_i^k(\tilde{\bx}_t).
\end{align*}
Now we need to show the normalization factor $\|\tilde{g}_i^k\|_{\tP}$, in order to eliminate the ``sup'' in the above expression with Lemma $\ref{GPathRel}$ and Theorem $\ref{thm: recursive_path_resHid}$. For any $i=1,2,\cdots,l$, $k=1,2,\cdots,m$, according to Lemma $\ref{RelG2Sca}$ and the positive homogeneity property of $\sigma_R(\cdot)$, we have
\begin{align*}
\tilde{g}_i^k(\tilde{\bx})&=\tilde{\sigma}(W_{i}^{k,:}\tilde{\bh}_{i-1})\\
&=\sum_{j=1}^J\alpha_j\sigma_R(\beta_jW_{i}^{k,:}\tilde{\bh}_{i-1}+\gamma_j)
\\&=\sum_{j=1}^J\alpha_j\sigma_R(|\beta_j|\|\tilde{g}_i^k\|_{\tP}\hat{\beta}_j\hat{W}_{i}^{k,:}\tilde{\bh}_{i-1}+\gamma_j) \qquad (|\hat{\beta}_j|=\|\hat{W}_{i}^{k,:}\tilde{\bh}_{i-1}\|_{\tP}=1)
\\&=\sum_{j=1}^J\alpha_j\left(|\beta_j|\|\tilde{g}_i^k\|_{\tP}+|\gamma_j|\right)\sigma_R(\hat{\beta}_j\hat{W}_{i}^{k,:}\tilde{\bh}_{i-1}+\hat{\gamma}_j). \qquad (|\hat{\beta}_j|+|\hat{\gamma}_j|=\|\hat{W}_{i}^{k,:}\tilde{\bh}_{i-1}\|_{\tP}=1)
\end{align*}
Here we simplify the notation by denoting $J:=J(i,k)$ and $\{\alpha_j,\beta_j,\gamma_j\}:=\{\alpha_j^{(i,k)},\beta_j^{(i,k)},\gamma_j^{(i,k)}\}$, only to be careful with the order of summation, i.e. keeping the indices $i,k$ before $j$ all the time. Therefore 
\begin{align*}
I_{2,1}&=\mathbb{E}_{\bxi}\sup_{\tilde{g}_{l+1}\in G_{l+1}^Q}\sum_{i=1}^l\sum_{k=1}^m \bw_{l+1}^TU_i^{:,k}\sum_{t=1}^n\xi_t\sum_{j=1}^J\alpha_j\left(|\beta_j|\|\tilde{g}_i^k\|_{\tP}+|\gamma_j|\right)\sigma_R(\hat{\beta}_j\hat{W}_{i}^{k,:}\tilde{\bh}_{i-1}+\hat{\gamma}_j)\\
&=\mathbb{E}_{\bxi}\sup_{\tilde{g}_{l+1}\in G_{l+1}^Q}\sum_{i=1}^l\sum_{k=1}^m \bw_{l+1}^TU_i^{:,k}\sum_{j=1}^J\alpha_j\left(|\beta_j|\|\tilde{g}_i^k\|_{\tP}+|\gamma_j|\right)\sum_{t=1}^n\xi_t\sigma_R(\hat{\beta}_j\hat{W}_{i}^{k,:}\tilde{\bh}_{i-1}+\hat{\gamma}_j)\\
&\le \mathbb{E}_{\bxi}\sup_{\tilde{g}_{l+1}\in G_{l+1}^Q}\sum_{i=1}^l\sum_{k=1}^m |\bw_{l+1}|^T|U_i^{:,k}|\sum_{j=1}^J|\alpha_j|\left(|\beta_j|\|\tilde{g}_i^k\|_{\tP}+|\gamma_j|\right)\left|\sum_{t=1}^n\xi_t\sigma_R(\hat{\beta}_j\hat{W}_{i}^{k,:}\tilde{\bh}_{i-1}+\hat{\gamma}_j)\right|\\
&\le \mathbb{E}_{\bxi}\sup_{\tilde{g}_{l+1}\in G_{l+1}^Q}\sum_{i=1}^l\sum_{k=1}^m |\bw_{l+1}|^T|U_i^{:,k}|\sum_{j=1}^J|\alpha_j|\left(|\beta_j|\|\tilde{g}_i^k\|_{\tP}+|\gamma_j|\right)\sup_{\substack{\|\hat{\bw}_{i}^{T}\tilde{\bh}_{i-1}\|_{\tP}\le 1\\|\hat{\beta}_j|+|\hat{\gamma}_j|\le 1}}\left|\sum_{t=1}^n\xi_t\sigma_R(\hat{\beta}_j\hat{\bw}_{i}^{T}\tilde{\bh}_{i-1}+\hat{\gamma}_j)\right|\\
&\le \mathbb{E}_{\bxi}\sup_{\tilde{g}_{l+1}\in G_{l+1}^Q}\sum_{i=1}^l\sum_{k=1}^m |\bw_{l+1}|^T|U_i^{:,k}|\sum_{j=1}^J|\alpha_j|\left(|\beta_j|\|\tilde{g}_i^k\|_{\tP}+|\gamma_j|\right)\sup_{\substack{\|\hat{\bw}_{l}^{T}\tilde{\bh}_{l-1}\|_{\tP}\le 1\\|\hat{\beta}|+|\hat{\gamma}|\le 1}}\left|\sum_{t=1}^n\xi_t\sigma_R(\hat{\beta}\hat{\bw}_{l}^{T}\tilde{\bh}_{l-1}+\hat{\gamma})\right|,
\end{align*}
where the last inequality is due to Lemma $\ref{RelG1}$. 

Write $\Xi_{l,n}(\bxi)=\sup\limits_{\substack{\|\hat{\bw}_{l}^{T}\tilde{\bh}_{l-1}\|_{\tP}\le 1\\|\hat{\beta}|+|\hat{\gamma}|\le 1}}\left|\sum\limits_{t=1}^n\xi_t\sigma_R(\hat{\beta}\hat{\bw}_{l}^{T}\tilde{\bh}_{l-1}+\hat{\gamma})\right|\ge 0$, we have 
\begin{align*}
I_{2,1}&\le \mathbb{E}_{\bxi}\sup_{\tilde{g}_{l+1}\in G_{l+1}^Q}\sum_{i=1}^l\sum_{k=1}^m |\bw_{l+1}|^T|U_i^{:,k}|\sum_{j=1}^J|\alpha_j|\left(|\beta_j|\|\tilde{g}_i^k\|_{\tP}+|\gamma_j|\right)\Xi_{l,n}(\bxi)\\
&\overset{(a)}{=}\mathbb{E}_{\bxi}\Xi_{l,n}(\bxi)\sup_{\tilde{g}_{l+1}\in G_{l+1}^Q}\sum_{i=1}^l\sum_{k=1}^m |\bw_{l+1}|^T|U_i^{:,k}|\sum_{j=1}^J|\alpha_j|\left[|\beta_j|(\|\tilde{g}_i^k\|_{\cP}+M_{i,k})+|\gamma_j|\right]\\
&\overset{(b)}{\le}\mathbb{E}_{\bxi}\Xi_{l,n}(\bxi)\sup_{\tilde{g}_{l+1}\in G_{l+1}^Q}\sum_{i=1}^l\sum_{k=1}^m |\bw_{l+1}|^T|U_i^{:,k}|(\gamma(\sigma)+\epsilon)(\|\tilde{g}_i^k\|_{\cP}+M_{i,k}+1)\\
&=(\gamma(\sigma)+\epsilon)\mathbb{E}_{\bxi}\Xi_{l,n}(\bxi)\sup_{\tilde{g}_{l+1}\in G_{l+1}^Q}\left\{\sum_{i=1}^l\sum_{k=1}^m (|\bw_{l+1}|^T|U_i^{:,k}|)\|\tilde{g}_i^k\|_{\cP}+\sum_{i=1}^l\sum_{k=1}^m |\bw_{l+1}|^T|U_i^{:,k}|(M_{i,k}+1)\right\}\\
&\overset{(c)}{\le}(\gamma(\sigma)+\epsilon)\mathbb{E}_{\bxi}\Xi_{l,n}(\bxi)\sup_{\tilde{g}_{l+1}\in G_{l+1}^Q}\left\{\frac{1}{c}\|\tilde{g}_{l+1}\|_{\cP}+|\bw_{l+1}|^T\sum_{i=1}^l |U_i|(\bM_{i}+\bm{1}_m)\right\}\\
&\overset{(d)}{=}\frac{(\gamma(\sigma)+\epsilon)}{c}\mathbb{E}_{\bxi}\Xi_{l,n}(\bxi)\sup_{\tilde{g}_{l+1}\in G_{l+1}^Q}\left\{\|\tilde{g}_{l+1}\|_{\cP}+M_{l+1,:}\right\}\\
&\overset{(e)}{=}\frac{(\gamma(\sigma)+\epsilon)}{c}\mathbb{E}_{\bxi}\Xi_{l,n}(\bxi)\sup_{\tilde{g}_{l+1}\in G_{l+1}^Q}\|\tilde{g}_{l+1}\|_{\tP}\\
&\le \frac{Q(\gamma(\sigma)+\epsilon)}{c}\mathbb{E}_{\bxi}\Xi_{l,n}(\bxi),
\end{align*}
where $(a)$ and $(e)$ come from the definition of $\|\tilde{g}_{l+1}\|_{\tP}$ (see $(\ref{Decompsition})$ in Theorem $\ref{thm: recursive_path_resHid}$), and $(b)$ is due to the bounded path norm of the two-layer ReLU network $\tilde{\sigma}$ (see $(\ref{2NNPthNrm})$ and $(\ref{CmPthNrm2NN})$), and $(c)$ uses (\ref{PathNormRel2}) in Lemma $\ref{GPathRel}$, and $(d)$ comes from the recursive definition of the modification vector (see $(\ref{PathResMHidi})$). 

\textbf{(2b-iii)~Bound for $\mathbb{E}_{\xi}\Xi_{l,n}(\bxi)$}

The last task is to bound $\mathbb{E}_{\xi}\Xi_{l,n}(\bxi)$. Since $0$ is in the set over which taking supreme, we have
\begin{align*}
\sup_{\substack{\|\hat{\bw}_{l}^{T}\tilde{\bh}_{l-1}\|_{\tP}\le 1\\|\hat{\beta}|+|\hat{\gamma}|\le 1}}\sum_{t=1}^n\xi_t\sigma_R(\hat{\beta}\hat{\bw}_{l}^{T}\tilde{\bh}_{l-1}+\hat{\gamma})\ge 0
\end{align*}
for any $\{\xi_t\}_{t=1}^n\in\{\pm 1\}^n$. Hence
\begin{align*}
&\sup_{\substack{\|\hat{\bw}_{l}^{T}\tilde{\bh}_{l-1}\|_{\tP}\le 1\\|\hat{\beta}|+|\hat{\gamma}|\le 1}}\left|\sum_{t=1}^n\xi_t\sigma_R(\hat{\beta}\hat{\bw}_{l}^{T}\tilde{\bh}_{l-1}+\hat{\gamma})\right|\\
\le \ & \max \left\{\sup_{\substack{\|\hat{\bw}_{l}^{T}\tilde{\bh}_{l-1}\|_{\tP}\le 1\\|\hat{\beta}|+|\hat{\gamma}|\le 1}}\sum_{t=1}^n\xi_t\sigma_R(\hat{\beta}\hat{\bw}_{l}^{T}\tilde{\bh}_{l-1}+\hat{\gamma}),\sup_{\substack{\|\hat{\bw}_{l}^{T}\tilde{\bh}_{l-1}\|_{\tP}\le 1\\|\hat{\beta}|+|\hat{\gamma}|\le 1}}\sum_{t=1}^n(-\xi_t)\sigma_R(\hat{\beta}\hat{\bw}_{l}^{T}\tilde{\bh}_{l-1}+\hat{\gamma})\right\}\\
\le \ &\sup_{\substack{\|\hat{\bw}_{l}^{T}\tilde{\bh}_{l-1}\|_{\tP}\le 1\\|\hat{\beta}|+|\hat{\gamma}|\le 1}}\sum_{t=1}^n\xi_t\sigma_R(\hat{\beta}\hat{\bw}_{l}^{T}\tilde{\bh}_{l-1}+\hat{\gamma})+\sup_{\substack{\|\hat{\bw}_{l}^{T}\tilde{\bh}_{l-1}\|_{\tP}\le 1\\|\hat{\beta}|+|\hat{\gamma}|\le 1}}\sum_{t=1}^n(-\xi_t)\sigma_R(\hat{\beta}\hat{\bw}_{l}^{T}\tilde{\bh}_{l-1}+\hat{\gamma}),
\end{align*}
which gives  
\begin{align*}
\mathbb{E}_{\bxi}\Xi_{l,n}(\bxi)&=\mathbb{E}_{\bxi}\sup_{\substack{\|\hat{\bw}_{l}^{T}\tilde{\bh}_{l-1}\|_{\tP}\le 1\\|\hat{\beta}|+|\hat{\gamma}|\le 1}}\left|\sum_{t=1}^n\xi_t\sigma_R(\hat{\beta}\hat{\bw}_{l}^{T}\tilde{\bh}_{l-1}+\hat{\gamma})\right|\\
&\le \mathbb{E}_{\bxi}\sup_{\substack{\|\hat{\bw}_{l}^{T}\tilde{\bh}_{l-1}\|_{\tP}\le 1\\|\hat{\beta}|+|\hat{\gamma}|\le 1}}\sum_{t=1}^n\xi_t\sigma_R(\hat{\beta}\hat{\bw}_{l}^{T}\tilde{\bh}_{l-1}+\hat{\gamma})+\mathbb{E}_{\xi}\sup_{\substack{\|\hat{\bw}_{l}^{T}\tilde{\bh}_{l-1}\|_{\tP}\le 1\\|\hat{\beta}|+|\hat{\gamma}|\le 1}}\sum_{t=1}^n(-\xi_t)\sigma_R(\hat{\beta}\hat{\bw}_{l}^{T}\tilde{\bh}_{l-1}+\hat{\gamma})\\
&=2\mathbb{E}_{\bxi}\sup_{\substack{\|\hat{\bw}_{l}^{T}\tilde{\bh}_{l-1}\|_{\tP}\le 1\\|\hat{\beta}|+|\hat{\gamma}|\le 1}}\sum_{t=1}^n\xi_t\sigma_R(\hat{\beta}\hat{\bw}_{l}^{T}\tilde{\bh}_{l-1}+\hat{\gamma}).
\end{align*}
Then using the contraction lemma (Lemma $\ref{RadContraction}$) and induction hypothesis to have
\begin{align*}
\mathbb{E}_{\bxi}\Xi_{l,n}(\bxi)&\le 2\mathbb{E}_{\bxi}\sup_{\substack{\|\hat{\bw}_{l}^{T}\tilde{\bh}_{l-1}\|_{\tP}\le 1\\|\hat{\beta}|+|\hat{\gamma}|\le 1}}\sum_{t=1}^n\xi_t(\hat{\beta}\hat{\bw}_{l}^{T}\tilde{\bh}_{l-1}+\hat{\gamma})\\&\le 2\mathbb{E}_{\bxi}\sup_{\substack{\|\hat{\bw}_{l}^{T}\tilde{\bh}_{l-1}\|_{\tP}\le 1\\|\hat{\beta}|+|\hat{\gamma}|\le 1}}\sum_{t=1}^n\xi_t\hat{\beta}\hat{\bw}_{l}^{T}\tilde{\bh}_{l-1}+2\mathbb{E}_{\xi}\sup_{\substack{\|\hat{\bw}_{l}^{T}\tilde{\bh}_{l-1}\|_{\tP}\le 1\\|\hat{\beta}|+|\hat{\gamma}|\le 1}}\sum_{t=1}^n\xi_t\hat{\gamma}\\
&\le 2\mathbb{E}_{\bxi}\sup_{\|\hat{\bw}_{l}^{T}\tilde{\bh}_{l-1}\|_{\tP}\le 1}\sum_{t=1}^n\xi_t\hat{\bw}_{l}^{T}\tilde{\bh}_{l-1}+2\mathbb{E}_{\xi}\sup_{|\hat{\gamma}|\le 1}\sum_{t=1}^n\xi_t\hat{\gamma}\\
&\le 2\mathbb{E}_{\bxi}\sup_{\|\tilde{g}'_{l}\|_{\tP}\le 1}\sum_{t=1}^n\xi_t\tilde{g}'_{l}+2\mathbb{E}_{\bxi}\sup_{\|\hat{\bu}\|_1+|\hat{\gamma}|\le 1}\sum_{t=1}^n\xi_t(\hat{\bu}^T\bx_t+\hat{\gamma})\\
&\overset{(f)}{\le} 2n\frac{1}{L_{\sigma}}\sqrt{\frac{2\ln{(2d+2)}}{n}}+2n\sqrt{\frac{2\ln{(2d+2)}}{n}}\\ 
&=2n\left(1+\frac{1}{L_{\sigma}}\right)\sqrt{\frac{2\ln{(2d+2)}}{n}},
\end{align*}
%\overset{(g)}{\le}
where we assume that $L_{\sigma}\neq 0$.\footnote{The case of $L_{\sigma}=0$ is trivial since it implies $\sigma(\cdot)$ is a constant. Therefore, for $l=1,2,\cdots,L$, $g_l=c\bm{1}_m$ is a single function, whose Rademacher complexity is obviously zero.}
\\ \hspace*{\fill} \\
\noindent\textbf{(3)~Final results}

Combining all above yields 
\begin{align*}
I_{2,1}\le n\frac{2Q(\gamma(\sigma)+\epsilon)}{c}\left(1+\frac{1}{L_{\sigma}}\right) \sqrt{\frac{2\ln{(2d+2)}}{n}},
\end{align*}
and 
\begin{align*}
I_{2}\le L_{\sigma}(I_{2,1}+I_{2,2})&\le
n\left(\frac{2Q(\gamma(\sigma)+\epsilon)(L_{\sigma}+1)}{c}\sqrt{\frac{2\ln{(2d+2)}}{n}}+\frac{L_{\sigma}Q}{c}\sqrt{\frac{2\ln{(2d+2)}}{n}}\right)\\
&= n\cdot\frac{2(\gamma(\sigma)+\epsilon)(L_{\sigma}+1)+L_{\sigma}}{c}\cdot Q\sqrt{\frac{2\ln{(2d+2)}}{n}}.
\end{align*}

Notice that to make the induction hypothesis $(f)$ still holds when $l\leftarrow l+1$, it is necessary to set 
\begin{align*}
c\ge 2(\gamma(\sigma)+\epsilon)(L_{\sigma}+1)+L_{\sigma},
\end{align*}
which gives
\begin{align*}
I_{2}\le nQ\sqrt{\frac{2\ln{(2d+2)}}{n}}.
\end{align*}
Therefore
\begin{align*}
\rad_n(G_{l+1}^Q)\le\frac{1}{n}(I_1+I_2)\le \epsilon+Q\sqrt{\frac{2\ln{(2d+2)}}{n}}.
\end{align*}
Since $\epsilon\in (0,1)$ is arbitrary, taking $\epsilon\rightarrow 0$ gives 
\begin{align*}
\hat{R}_n(G_{l+1}^Q)\le Q\sqrt{\frac{2\ln{(2d+2)}}{n}}.
\end{align*}

In fact, recall the above restriction that $c\ge L_{\sigma}$ when $l=1$ (see Appendix \ref{sec:ProofRad} \textbf{(1)}), we can just take 
\begin{equation*}
c=c_{\sigma}>2\gamma(\sigma)(L_{\sigma}+1)+L_{\sigma}:=\tilde{c}_{\sigma}.
\end{equation*}
That is to say, for any $c>\tilde{c}_{\sigma}$ or $0<c-\tilde{c}_{\sigma}\ll 1$ in principle, it is sufficient to take $\epsilon\le \frac{c-\tilde{c}_{\sigma}}{2(L_{\sigma}+1)}$, in order to have $c\ge2(\gamma(\sigma)+\epsilon)(L_{\sigma}+1)+L_{\sigma}$ and all the proof above holds.

Based on the control for the Rademacher complexity of $G_1^Q,G_2^Q,\cdots,G_L^Q$, for $\mathcal{F}_Q$ ($l=L+1$) we can similarly obtain
\begin{align*}
I_1&\le n\epsilon,\\
I_{2,1}&\le n\cdot 2Q(\gamma(\sigma)+\epsilon)\left(1+\frac{1}{L_{\sigma}}\right) \sqrt{\frac{2\ln{(2d+2)}}{n}},\\
I_{2,2}&\le n\cdot Q\sqrt{\frac{2\ln{(2d+2)}}{n}},\\
I_2&\le I_{2,1}+I_{2,2} = n\cdot \left[2(\gamma(\sigma)+\epsilon)\left(1+\frac{1}{L_{\sigma}}\right)+1\right]\cdot Q\sqrt{\frac{2\ln{(2d+2)}}{n}}.
\end{align*}
Here we still require that the weight parameter $c>\tilde{c}_{\sigma}$, and the tolerance error $\epsilon\le \frac{c-\tilde{c}_{\sigma}}{2(L_{\sigma}+1)}$. Therefore
\begin{align*}
\rad_n(\mathcal{F}_Q)\le\frac{1}{n}(I_1+I_2)\le \epsilon+\left[2(\gamma(\sigma)+\epsilon)\left(1+\frac{1}{L_{\sigma}}\right)+1\right]\cdot Q\sqrt{\frac{2\ln{(2d+2)}}{n}}.
\end{align*}
Again, since $\epsilon\in (0,1)$ is arbitrary, taking $\epsilon\rightarrow 0$ yields
\begin{align*}
\rad_n(\mathcal{F}_Q)\le
\left[2\gamma(\sigma)\left(1+\frac{1}{L_{\sigma}}\right)+1\right]\cdot Q\sqrt{\frac{2\ln{(2d+2)}}{n}}=c_\sigma^*Q\sqrt{\frac{2\ln{(2d+2)}}{n}},
\end{align*}
where $c_\sigma^*:=\tilde{c}_{\sigma}/L_{\sigma}$. The proof is completed.
\end{proof}

\begin{remark}
We certainly have to assume $L_{\sigma}\neq 0$ in the above proof, but this trivial case can be actually handled implicitly. In fact, we can perform a similar analysis on $\tilde{g}_l'$ instead of $\tilde{g}_l$ for $l=1,2,\cdots,L$, which does not need to assume that $L_{\sigma}\neq 0$. All the proof above holds if we replace $L_{\sigma}$ by $1$ (since no activation after the linear transformation is equivalent to the identity mapping)\footnote{Let the induction hypothesis (\ref{RadG}) holds for $(G_l^Q)'$. For the error term $I_1$, the analysis certainly remains the same; for the main term $I_2$, the analysis begins from $(\ref{Activ})/L_{\sigma}$, and also remains the same in the following. Therefore all the proof holds after replacing $L_{\sigma}$ by $1$.}. In this case, we have $c_\sigma^*=\tilde{c}_{\sigma}=4\gamma(\sigma)+1$, and it is still need to take $c=c_{\sigma}>\tilde{c}_{\sigma}$. 
\end{remark}

\subsection{Proof of Lemma~\ref{thm: recursive_path_res}}\label{prf-recursive_path_res}

\begin{proof}
The proof of Lemma \ref{thm: recursive_path_res} is the same as that of Theorem \ref{thm: recursive_path_resHid} (see Remark \ref{rmk:prfmodsame} and Appendix \ref{sec:ExtHidNeu}).
\end{proof}

\subsection{Proof of Theorem~\ref{thm: path_res_control}}\label{prf-path_res_control}

\begin{proof}
The proof comes with induction. We firstly have 
\begin{align*}
\bM_{1}&=\bm{0}_m \le c|W_{1}|\bm{1}_m,\\
\bM_2&=c|W_2||U_1|\bm{1}_m \le c|W_2|(I+|U_1|)\bm{1}_m,
\end{align*}
which implies $(\ref{ModControll})$ for $l=1,2$. Here the ``$\le$'' holds in an entry-wise sense.

Denote $Z_k=(I+|U_{k}|)(I+c|W_{k}|)\cdots(I+|U_2|)(I+c|W_2|)(I+|U_1|)$ for $k\ge 2$, and $Z_1=I+|U_1|$, then the aim is to prove
\begin{align*}
\bM_{l+1}&\le c|W_{l+1}|Z_l\bm{1}_m, 
\end{align*}
using the assumption that $\bM_{k}\le c|W_{k}|Z_{k-1} \bm{1}_m$, $k=1,2,\cdots,l$. 

Notice that $Z_k\ge I$ for all $k\ge 1$, and 
\begin{align}
Z_k&=(I+|U_{k}|)(I+c|W_{k}|)Z_{k-1}\nonumber\\
&\ge Z_{k-1}+c|U_{k}||W_{k}|Z_{k-1}+|U_{k}|Z_{k-1}\nonumber\\
&\ge Z_{k-1}+c|U_{k}||W_{k}|Z_{k-1}+|U_{k}|\label{BndResZ}
\end{align}
for all $k\ge 2$, we have $c|U_{k}||W_{k}|Z_{k-1}\le (Z_k-Z_{k-1})-|U_k|$ for all $k\ge 2$. Therefore, for $l=1,2,\cdots,L-1$, by (\ref{PathResMHid}) and (\ref{BndResZ}), we have
\begin{align*}
\bM_{l+1}&=c|W_{l+1}|\sum_{k=2}^l|U_k|\bM_k+c|W_{l+1}|\sum_{k=1}^l|U_k|\bm{1}_m \qquad (\bM_1=0) \\
&\le c|W_{l+1}|\sum_{k=2}^lc|U_k||W_{k}|Z_{k-1} \bm{1}_m+c|W_{l+1}|\sum_{k=1}^l|U_k|\bm{1}_m
\\&\le c|W_{l+1}|\sum_{k=2}^l[(Z_k-Z_{k-1})-|U_k|] \bm{1}_m+c|W_{l+1}|\sum_{k=1}^l|U_k|\bm{1}_m
\\&=c|W_{l+1}|\left(\sum_{k=2}^l(Z_k-Z_{k-1}) \right)\bm{1}_m+\left(c|W_{l+1}|\sum_{k=1}^l|U_k|\bm{1}_m-c|W_{l+1}|\sum_{k=2}^l|U_k|\bm{1}_m\right)
\\&=c|W_{l+1}|(Z_l-Z_1)\bm{1}_m+c|W_{l+1}||U_1|\bm{1}_m
\\&\le c|W_{l+1}|(Z_l-Z_1)\bm{1}_m+c|W_{l+1}|Z_1\bm{1}_m
\\&=c|W_{l+1}|Z_l\bm{1}_m,
\end{align*}
which gives $(\ref{ModControll})$. The proof of $(\ref{ModControlL})$ is similar. 
\end{proof}

\subsection{Proof of Proposition \ref{pro: approx-resnet}}\label{prf-approx-resnet}

\begin{proof}
According to Theorem $\ref{2-NNApp}$, there exists a two-layer neural network with width $Lm$, such that
\begin{align}
    \mathbb{E}_{\bx}\left[\sum_{k=1}^{Lm}a_k\sigma(\bw_k^T\tilde{\bx})-f^*(\bx)\right]^2&\le \frac{3C_{\sigma}\|f\|^2_{\cB}}{Lm},\label{apperrRes}\\
    \sum_{k=1}^{Lm}|a_k|(\|\bw_k\|_1+1)&\le 2\|f\|_{\cB}.\label{pthnrmRes}
\end{align}
Now we construct a residual network $f(\bx;\tilde{\theta})$ with input dimension $d+1$, depth $L$, width $m$ and $D=d+2$ by selecting the parameters to be 
\begin{align*}
V&=[I_{d+1} \ 0]^T,\quad \bm{\alpha}=e_D=[0,\cdots,0,1]^T,\\
W_l&=\left[\begin{array}{cc}
   \bw_{(l-1)m+1}^T  &  0\\
   \bw_{(l-1)m+2}^T   &  0\\
    \vdots & \vdots\\
    \bw_{lm}^T & 0 \\
\end{array}\right],\quad
U_l=\left[\begin{array}{cccc}
   0  &  0 & \cdots & 0\\
   \vdots  &  \vdots & \ddots & \vdots\\
   0  &  0 & \cdots & 0\\
   a_{(l-1)m+1}  &  a_{(l-1)m+2} & \cdots & a_{lm}\\
\end{array}\right] 
\end{align*}
for $l=1,2,\cdots,L$. It is easy to verify that
\begin{align}\label{2-NNRes}
f(\bx;\tilde{\theta})=\sum_{k=1}^{Lm}a_k\sigma(\bw_k^T\tilde{\bx}),
\end{align}
and the  weighted path norm of residual network $f(\bx;\tilde{\theta})$ 
\begin{align}\label{2-NNResPthNrm}
\|\tilde{\theta}\|_{\cP}=c\sum_{k=1}^{Lm}|a_k|\|\bw_k\|_1,
\end{align}
where $c$ can be taken as $c=c^*_{\sigma}$. Combining $(\ref{apperrRes})$ and $(\ref{2-NNRes})$ yields $(\ref{apperr})$.

We now need to compute the modified weighted path norm $\|\tilde{\theta}\|_{\tP}=\|\tilde{\theta}\|_{\cP}+r$. Notice that $|W_i||U_j|=\bm{0}_{m\times m}$ for all $i,j=1,2,\cdots,L$, by (\ref{PathResMHid}) we have 
\begin{align*}
\bM_{l+1}=c\sum_{k=1}^l|W_{l+1}||U_k|(\bM_k+\bm{1}_m)=\bm{0}_m
\end{align*}
for $l=1,2,\cdots,L-1$. Therefore, by (\ref{PathResM}), 
\begin{align}\label{2-NNResMod}
r=|\bm{\alpha}|^T\sum_{l=1}^L|U_l|\bm{1}_m=\sum_{l=1}^L\||\bm{\alpha}|^T|U_l|\|_1=\sum_{l=1}^L\sum_{i=1}^m|a_{(l-1)m+i}|=\sum_{k=1}^{Lm}|a_k|.
\end{align}
Combining $(\ref{2-NNResPthNrm})$ and $(\ref{2-NNResMod})$ gives 
\begin{align}\label{2-NNResPthNrmMod}
\|\tilde{\theta}\|_{\tP}\le\max\{c,1\}\sum_{k=1}^{Lm}|a_k|(\|\bw_k\|_1+1).
\end{align}
Taking $c=c^*_{\sigma}$, and combining $(\ref{pthnrmRes})$ and $(\ref{2-NNResPthNrmMod})$ yield $(\ref{pthnrm})$. The proof is completed.
\end{proof}

\subsection{Proof of Proposition \ref{pro: apriori-resnet}}\label{prf-apriori-resnet}

\begin{proof}
The proof is almost the same as Theorem \ref{thm: apriori-two-layer} (see Appendix \ref{sec: prf-apriori-two-layer}). To obtain (\ref{ResNetPriorEst}), We just need to take $C'_{\sigma}=4\gamma(\sigma)+1$ by Theorem \ref{RadResNet} in $(\ref{Rad2Res})$, and apply Proposition \ref{pro: approx-resnet} ((\ref{apperr}) and (\ref{pthnrm})) to bound $(\ref{InterRhat})$. The proof is completed. 
\end{proof}

\end{document}